\documentclass[twoside]{article}
\usepackage[accepted]{aistats2018}
\usepackage{enumitem}

\usepackage{mathtools}
\usepackage[utf8]{inputenc}
\usepackage[english]{babel}
\usepackage{amssymb,amsmath,amsthm}
\usepackage{color}

\newtheorem{theorem}{Theorem}
\newtheorem{lemma}{Lemma}
\newtheorem{corollary}{Corollary}
\newtheorem{proposition}{Proposition}
\newtheorem{definition}{Definition}

\DeclareMathOperator*{\argmin}{arg\,min}
\DeclareMathOperator*{\argmax}{arg\,max}
\DeclarePairedDelimiter{\ceil}{\lceil}{\rceil}

\newcommand{\norm}[1]{\left\lVert#1\right\rVert}

\newcommand{\OPT}{\textup{OPT}}

\def\l({\left(}
\def\r){\right)}
\def\bl({\Big(}
\def\br){\Big)}
\def\beq{\begin{equation}}
\def\eeq{\end{equation}}

\usepackage{tikz,pgfplots}
\def\x{{\mathbf x}}
\def\y{{\mathbf y}}

\def\z{{\mathbf z}}

\def\X{{\mathbf X}}

\def\m{{\mathbf m}}

\def\A{{\mathbf A}}
\def\B{{\mathbf B}}

\def\v{{\mathbf v}}

\def\K{{\mathbf K}}
\def\I{{\mathbf I}}

\def\bbeta{\mathbf{\beta}}
\def\bomega{\mathbf{\omega}}

\def\I{{ \mathbf I }}

\def\bbeta{{\boldsymbol{\beta}}}

\def\a{{\mathbf{a}}}

\usepackage{algorithm}
\usepackage[noend]{algpseudocode}
\usepackage{graphicx, subfig}

\newcommand{\RALG}{\textsc{Oblivious-Greedy}}
\newcommand{\PRO}{\textsc{PRo-GREEDY}}
\newcommand{\GREEDY}{\textsc{Greedy}}
\newcommand{\opt}{{\operatorname{OPT}}}
\newcommand{\Obl}{\textsc{Oblivious}}
\newcommand{\SGREEDY}{\textsc{Stochastic-Greedy}}
\newcommand{\RGREEDY}{\textsc{Random-Greedy}}
\newcommand{\ORLIN}{\textsc{OSU}}
\newtheorem{remark}{Remark}

\newcommand{\OMP}{\textsc{Orthogonal-Matching-Pursuit}}

\begin{document}

%

%

\twocolumn[

\aistatstitle{Robust Maximization of Non-Submodular Objectives}

\aistatsauthor{ Ilija Bogunovic\textsuperscript{\dag} \And Junyao Zhao\textsuperscript{\dag} \And Volkan Cevher}

\aistatsaddress{ LIONS, EPFL \\ ilija.bogunovic@epfl.ch \And  LIONS, EPFL \\ zhaoju@student.ethz.ch \And LIONS, EPFL \\ volkan.cevher@epfl.ch  } 
]

{
  \renewcommand{\thefootnote}%
    {\fnsymbol{footnote}}
  \footnotetext[2]{Equal contribution.}
}

\begin{abstract}
  We study the problem of maximizing a monotone set function subject to a cardinality constraint $k$ in the setting where some number of elements $\tau$ is deleted from the returned set. The focus of this work is on the worst-case adversarial setting. While there exist constant-factor guarantees when the function is submodular~\cite{orlin2016robust,bogunovic2017robust}, there are no guarantees for non-submodular objectives. In this work, we present a new algorithm \RALG~and prove the first constant-factor approximation guarantees for a wider class of non-submodular objectives. The obtained theoretical bounds are the \emph{first} constant-factor bounds that also hold in the \emph{linear regime}, i.e.~when the number of deletions $\tau$ is linear in $k$. Our bounds depend on established parameters such as the \emph{submodularity ratio} and some novel ones such as the \emph{inverse curvature}. We bound these parameters for two important objectives including support selection and variance reduction. Finally, we numerically demonstrate the robust performance of \RALG~for these two objectives on various datasets.
\end{abstract}

\section{Introduction}
A wide variety of important problems in machine learning can be formulated as the maximization of a \emph{monotone}\footnote{
Non-negative and normalized (i.e. $f(\emptyset) = 0$) $f(\cdot)$ is \emph{monotone} if for any sets $X \subseteq Y \subseteq V$ it holds $f(X) \leq f(Y)$.} 
set function $f:2^V \to \mathbb{R_{+}}$ under the cardinality constraint $k$, i.e.
\begin{equation}
	\label{eq:cardinality_maximization}
	\max_{S \subseteq V, |S| \leq k} f(S),
\end{equation}
where $V= \lbrace v_1, \cdots v_n \rbrace$ is the ground set of items. However, in many applications, we might require robustness of the solution set, meaning that the objective value should deteriorate as little as possible after a subset of elements is deleted. 

For example, an important problem in machine learning is feature selection, where the goal is to extract a subset of features that are informative w.r.t.~a given task (e.g.~classification).
For some tasks, it is of great importance to select features that exhibit robustness against deletions.~This is particularly important in domains with non-stationary feature distributions or with input sensor failures~\cite{globerson2006nightmare}. Another important example is the optimization of an unknown function from point evaluations that require performing costly experiments.~When the experiments can fail, protecting against worst-case failures becomes important.

In this work, we consider the following robust variant of Problem \eqref{eq:cardinality_maximization}:
\begin{equation}
	\label{eq:robust_cardinality_maximization}
	\max_{S \subseteq V, |S| \leq k} \; \min_{E \subseteq S, |E| \leq \tau} f(S \setminus E),
\end{equation}
where\footnote{When $\tau = 0$, Problem~\eqref{eq:robust_cardinality_maximization} reduces to Problem~\eqref{eq:cardinality_maximization}.} $\tau$ is the size of subset $E$ that is removed from the solution set $S$. When the objective function exhibits \textit{submodularity}, a natural notion of diminishing returns\footnote{$f(\cdot)$ is submodular if for any sets $X \subseteq Y \subseteq V$ and any element $e \in V \setminus Y$, it holds that
$
    f(X \cup \lbrace e \rbrace) - f(X) \ge f(Y \cup \lbrace e \rbrace) - f(Y).
$}, a constant factor approximation guarantee can be obtained for the robust Problem 2 \cite{orlin2016robust,bogunovic2017robust}. However, in many applications such as the above mentioned feature selection problem, the objective function $f(\cdot)$ is not submodular and the obtained guarantees are not applicable.

\textbf{Background and related work.}~When the objective function is submodular, the simple \GREEDY~algorithm~\cite{nemhauser1978analysis} achieves a $(1 - 1/e)$-multiplicative approximation guarantee for Problem~\ref{eq:cardinality_maximization}. The constant factor can be further improved by exploiting the properties of the objective function, such as the \textit{closeness} to being modular captured by the notion of \textit{curvature} \cite{conforti1984submodular,vondrak2010submodularity,iyer2013curvature}. 

In many cases, the \GREEDY~algorithm performs well empirically even when the objective function deviates from being submodular. An important class of such objectives are $\gamma$-\textit{weakly} submodular functions. Simply put, \textit{submodularity ratio} $\gamma$ is a quantity that characterizes how \textit{close} the function is to being submodular.
It was first introduced in \cite{das2011submodular}, where it was shown that for such functions the approximation ratio of \GREEDY~for Problem~\eqref{eq:cardinality_maximization} degrades slowly as the submodularity ratio decreases i.e.~as $(1 - e^{-\gamma})$. In~\cite{bian17guarantees}, the authors obtain the approximation guarantee of the form $\alpha^{-1} (1-e^{-\gamma \alpha})$, that further depends on the curvature $\alpha$.

When the objective is submodular, the \GREEDY~algorithm can perform arbitrarily badly when applied to Problem~\eqref{eq:robust_cardinality_maximization}~\cite{orlin2016robust,bogunovic2017robust}. A submodular version of Problem~\eqref{eq:robust_cardinality_maximization} was first introduced in Krause~\emph{et al.}~\cite{krause2008robust}, while the first efficient algorithm and constant factor guarantees were obtained in Orlin~\emph{et al.}~\cite{orlin2016robust} for $\tau = o(\sqrt{k})$. In Bogunovic~\emph{et al.}~\cite{bogunovic2017robust}, the authors introduce the \PRO~algorithm that attains the same $0.387$-guarantee but it allows for greater robustness, i.e.~the allowed number of removed elements is $\tau = o(k)$. It is not clear how the obtained guarantees generalize for non-submodular functions.


One important class of non-submodular functions that we consider in this work are those used for support selection:
\begin{equation}
	\label{eq:f_support}
	f(S) := \max_{\x \in \mathcal{X}, \text{supp}(\x) \subseteq S} l(\x),
\end{equation}
where $l(\cdot)$ is a continuous function, $\mathcal{X}$ is a convex set and $\text{supp}(\x) = \lbrace i : x_i \neq 0 \rbrace $. 
A popular way to solve the problem of finding a $k$-sparse vector that maximizes $l$, i.e.~$ \x \in \argmax_{\x \in \mathcal{X}, \| \x \|_0 \leq k} l(\x)$ is to maximize the auxiliary set function in \eqref{eq:f_support} subject to the cardinality constraint $k$. This setting and its variants have been used in various applications, for example, sparse approximation~\cite{das2011submodular, cevher2011greedy}, feature selection~\cite{khanna2017scalable}, sparse recovery~\cite{candes2006stable}, sparse M-estimation~\cite{jain2014iterative} and column subset selection problems~\cite{altschuler2016greedy}.~An important result from~\cite{elenberg2016restricted} states that if $l(\cdot)$ is $(m, L)$-(strongly concave, smooth) then $f(S)$ is weakly submodular with submodularity ratio $\gamma \geq \frac{m}{L}$. Consequently, this result enlarges the number of problems where \GREEDY~comes with guarantees. In this work, we consider the robust version of this problem, where the goal is to protect against the worst-case adversarial deletions of features. 

Deletion robust submodular maximization in the streaming setting has been considered in~\cite{mitrovic2017streaming,mirzasoleiman2017deletion,kazemi2017deletion}. Other versions of robust submodular optimization problems have also been studied. In~\cite{krause2008robust}, the goal is to select a set of elements that is robust against the worst possible objective from a given finite set of monotone submodular functions. The same problem with different types of constraints is considered in~\cite{powersconstrained}. It was further studied in the domain of influence maximization~\cite{he2016robust, chen2016robust}. The robust version of the budget allocation problem was considered in~\cite{staibrobust}.
In~\cite{hassidim2016submodular}, the authors study the problem of maximizing a monotone submodular function under adversarial noise. We conclude this section by noting that very recently a couple of different works have further studied robust submodular problems~\cite{udwani2017multi,wilder2017equilibrium,anari2017robust,chen2017robust}.

\textbf{Main contributions:}
\begin{itemize}
\item We initiate the study of the robust optimization Problem~\eqref{eq:robust_cardinality_maximization} for a wider class of monotone non-submodular functions. We present a new algorithm \RALG~and prove the first constant factor approximation guarantees for Problem~\eqref{eq:robust_cardinality_maximization}. When the function is submodular and under mild conditions, we recover the approximation guarantees obtained in the previous works~\cite{orlin2016robust, bogunovic2017robust}. 

\item For both non-submodular and submodular case, we obtain the \emph{first} constant factor approximation guarantees for the \emph{linear regime}, i.e. when $\tau = ck$ for some $c \in (0,1)$.

\item Our theoretical bounds are expressed in terms of parameters that further characterize a set function. Some of them have been used in previous works, e.g. submodularity ratio, and some of them are novel, such as the \emph{inverse curvature}. We prove some interesting relations between these parameters and obtain theoretical bounds for them in two important applications:~(i) support selection and~(ii) variance reduction objective used in batch Bayesian optimization. This allows us to obtain the \emph{first robust} guarantees for these two important objectives.

\item Finally, we experimentally validate the robustness of \RALG~in several scenarios, and demonstrate that it outperforms other robust and non-robust algorithms.
\end{itemize}
\section{Preliminaries}
\label{set_ratios_section}
\textbf{Set function ratios.}~In this work, we consider a normalized monotone set function $f:2^V \to \mathbb{R}_+$; we proceed by defining several quantities that characterize it. Some of the quantities were introduced and used in various different works, while the novel ones that we consider are \textit{inverse curvature, bipartite supermodularity ratio and (super/sub)additivity ratio}.  

\begin{definition}[Submodularity~\cite{das2011submodular} and Supermodularity ratio]
The submodularity ratio of $f(\cdot)$
is the largest scalar $\gamma \in [0,1]$ s.t.
\begin{equation}
\label{eq:submodularity_ratio}
		\frac{\sum_{i \in \Omega} f(\lbrace i \rbrace | S)}{f(\Omega | S)} \geq \gamma, \quad \forall \textrm{ disjoint }S, \Omega \subseteq V.
\end{equation}
while the supermodularity ratio is the largest scalar $\check{\gamma} \in [0, 1]$ s.t.
\begin{equation}
\label{eq:supermodularity_ratio}
		\frac{f(\Omega | S)}{\sum_{i \in \Omega} f(\lbrace i \rbrace | S)} \geq \check{\gamma}, \quad \forall \textrm{ disjoint }S, \Omega \subseteq V.
\end{equation}
\end{definition}
The function $f(\cdot)$ is submodular (supermodular) \textit{iff} $\gamma=1$ ($\check{\gamma} = 1$). Hence, the submodularity/supermodularity ratio measures to what extent the function has submodular/supermodular properties. 
While $f(\cdot)$ is modular \textit{iff} $\gamma = \check{\gamma} = 1$, in general, $\gamma$ can be different from $\check{\gamma}$.

\begin{definition}[Generalized curvature~\cite{vondrak2010submodularity, bian17guarantees} and inverse generalized curvature]
The generalized curvature of $f(\cdot)$
is the smallest scalar $\alpha \in [0,1]$ s.t.
\begin{equation}
\label{eq:curvature}
		\frac{f( \lbrace i \rbrace | S \setminus \lbrace i \rbrace \cup \Omega)}{f(\lbrace i \rbrace| S \setminus \lbrace i \rbrace)} \geq  1 - \alpha, \quad \forall S, \Omega \subseteq V, i \in S \setminus \Omega,
\end{equation}
while the inverse generalized curvature is the smallest scalar $\check{\alpha} \in [0,1]$ s.t.
\begin{equation}
\label{eq:inverse_curvature}
		\frac{f(\lbrace i \rbrace| S \setminus \lbrace i \rbrace)}{f( \lbrace i \rbrace | S \setminus \lbrace i \rbrace \cup \Omega)} \geq  1 - \check{\alpha}, \quad \forall S, \Omega \subseteq V, i \in S \setminus \Omega.
\end{equation}
\end{definition}
The function $f(\cdot)$ is submodular (supermodular) \textit{iff} $\check{\alpha} = 0$ ($\alpha = 0$). The function is modular \textit{iff} $\alpha = \check{\alpha} = 0$. In general, $\alpha$ can be different from $\check{\alpha}$.

\begin{definition}[sub/superadditivity ratio]
	The subadditivity ratio of $f(\cdot)$ is the largest scalar $\nu \in [0,1]$ such that 
	\begin{equation}
		\label{eq:subadditivity}
		\frac{\sum_{i \in S} f(\lbrace i \rbrace)}{f(S)} \geq \nu, \quad \forall S \subseteq V.
	\end{equation}
	The superadditivity ratio is the largest scalar $\check{\nu} \in [0,1]$ such that 
	\begin{equation} \label{eq:superadditivity}
		\frac{f(S)}{\sum_{i \in S} f(\lbrace i \rbrace)} \geq \check{\nu}, \quad \forall S \subseteq V.
	\end{equation}
\end{definition}
If the function is submodular (supermodular) then $\nu = 1$ ($\check{\nu}=1$).
\par
The following proposition captures the relation between the above quantities.
\begin{proposition}\label{prop:const_relation}
For any $f(\cdot)$, the following relations hold:
\begin{equation*}
	\nu \geq \gamma \geq 1 - \check{\alpha} \quad \text{and} \ \check{\nu} \ge \check{\gamma} \ge 1 - \alpha.
\end{equation*}
\end{proposition}

We also provide a more general definition of the bipartite subadditivity ratio used in~\cite{khanna2017scalable}.
\begin{definition}[Bipartite subadditivity ratio]\label{eq:two_separate_subbaditivity}
	The bipartite subadditivity ratio of $f(\cdot)$ is the largest scalar $\theta \in [0,1]$ s.t.
	\begin{equation} \label{eq:2_separate_sa}
		\frac{f(A) + f(B)}{f(S)} \geq \theta, \quad \forall S \subseteq V, A \cup B = S, A \cap B = \emptyset.
	\end{equation}
\end{definition}
\begin{remark}
\label{theta_nu_inv_nu}
For any $f(\cdot)$, it holds that $\theta \geq \check{\nu} \nu$.
\end{remark}

\par{{\textbf{Greedy guarantee.}}
Different works~\cite{das2011submodular, bian17guarantees} have studied the performance of the \GREEDY~algorithm~\cite{nemhauser1978analysis} for Problem~\ref{eq:cardinality_maximization} when the objective is $\gamma$-weakly submodular.
In our analysis, we are going to make use of the following important result from~\cite{das2011submodular}. 
\begin{lemma}
\label{lemma:greedy}
For a monotone normalized set function $f:2^V \to \mathbb{R_{+}}$, with submodularity ratio $\gamma \in [0,1]$ the \GREEDY~algorithm when run for $l$ steps returns a set $S_l$ of size $l$ such that
\[
	f(S_l) \geq  \left(1 - e^{-\gamma \frac{l}{k}}\right) f(\OPT_{(k,V)}),
\]
where $\OPT_{(k,V)}$ is used to denote the optimal set of size $k$, i.e., $\OPT_{(k,V)} \in \argmax_{S \subseteq V, |S| \leq k} f(S)$. 
\end{lemma}

\section{Algorithm and its Guarantees}
\label{section_algorithm_and_guarantees}
We present our $\RALG$~algorithm in Algorithm~\ref{algorithm:alg}.~The algorithm requires a non-negative monotone set function $f:2^V \to \mathbb{R}_{+}$, and the ground set of items $V$. It constructs two sets $S_0$ and $S_1$. The first set $S_0$ is constructed via oblivious selection, i.e.~$\lceil \beta \tau \rceil$ items with the individually highest objective values are selected. Here, $\beta \in \mathbb{R}_+$ is an input parameter, that together with $\tau$, determines the size of $S_0$ ($|S_0| = \lceil \beta \tau \rceil \leq k$). We provide more information on this parameter in the next section. The second set $S_1$, of size $k - |S_0|$, is obtained by running the $\GREEDY$ algorithm on the remaining items $V \setminus S_0$. Finally, the algorithm outputs the set $S = S_0 \cup S_1$ of size $k$ that is robust against the worst-case removal of $\tau$ elements.

Intuitively, the role of $S_0$ is to ensure robustness, as its elements are selected independently of each other and have high marginal values, while $S_1$ is obtained greedily and it is near-optimal on the set $V \setminus S_0$. 


\RALG~is simpler than the submodular algorithms \PRO~\cite{bogunovic2017robust} and \ORLIN~\cite{orlin2016robust}. 
Both of these algorithms construct multiple sets (buckets) whose number and size depend on the input parameters $k$ and $\tau$. In contrast, \RALG~always constructs two sets, where the first set is obtained by the fast $\Obl$ selection.

\begin{algorithm}[t!]
    \caption{$\RALG$ algorithm \label{algorithm:alg}}
    \begin{algorithmic}[1]
      \Require  Set $V$, $k$, $\tau$, $\beta \in \mathbb{R}_+ \text{ and } \lceil \beta \tau \rceil \leq k$
       \Ensure Set $S \subseteq V$ such that $|S| \leq k$
            \State $S_0, S_1 \gets \emptyset$
            \For {$ i \gets 0 \textbf{ to } 	\lceil \beta \tau \rceil$}
                \State  $ v \gets \argmax_{v \in V \setminus S_0} f(\lbrace v \rbrace)$
                \State $S_0 \gets S_0 \cup \lbrace v \rbrace$
            \EndFor
        \State $S_1 \gets \GREEDY (k - |S_0|,\ (V \setminus S_0))$
    \State {$S \gets S_0 \cup S_1$}\\
    \Return $S$ 
    \end{algorithmic}
\end{algorithm}
 
For Problem~\eqref{eq:cardinality_maximization} and the weakly submodular objective, the $\GREEDY$ algorithm achieves a constant factor approximation (Lemma~\ref{lemma:greedy}), while $\Obl$ selection achieves $(\gamma / k)$-approximation~\cite{khanna2017scalable}. For the harder Problem~\eqref{eq:robust_cardinality_maximization}, $\GREEDY$ can fail arbitrarily badly~\cite{bogunovic2017robust}. Interestingly enough, the combination of these two algorithms reflected in $\RALG$ leads to a constant factor approximation for Problem~\eqref{eq:robust_cardinality_maximization}.

\subsection{Approximation guarantee}
The quantity of interest in this section is the remaining utility after the adversarial removal of elements $f(S \setminus E_S^*)$, where $S$ is the set of size $k$ returned by \RALG, and $E_S^*$ is the set of size $\tau$ chosen by the adversary, i.e.,
$
	E_S^* \in \argmin_{E \subset S, |E| \leq \tau} f(S\setminus E).
$
Let $\opt_{(k - \tau, V \setminus E_S^*)}$ denote the optimal solution, of size $k - \tau$, when the ground set is $V \setminus E_S^*$. The goal in this section is to compare $f(S \setminus E_S^*)$ to $f(\opt_{(k - \tau, V \setminus E_S^*)})$.\footnote{As shown in~\cite{orlin2016robust}, $ f(\opt_{(k - \tau, V \setminus E_S^*)}) \geq f(\opt \setminus E_{\opt}^*)$, where $\opt$ is the optimal solution to Problem~\eqref{eq:robust_cardinality_maximization}. } All the omitted proofs from this section can be found in the supplementary material.
\par
\textbf{Intermediate results.}
Before stating our main result, we provide three lower bounds on $f(S \setminus E_S^*)$. For the returned set $S = S_0 \cup S_1$, we let $E_0$ denote elements removed from $S_0$, i.e., $E_0:= E_S^* \cap S_0 $ and similarly $E_1:= E_S^* \cap S_1$.
The first lemma is borrowed from~\cite{bogunovic2017robust}, and states that $f(S \setminus E_S^*)$ is at least some constant fraction of the utility of the elements obtained greedily in the second stage.
\begin{lemma}\label{lemma:mu}
For any $f(\cdot)$ (not necessarily submodular), let $\mu \in [0,1]$ be a constant such that $f(E_1\ |\ (S \setminus E_S^*)) = \mu f(S_1)$ holds. Then, $f(S \setminus E_S^*) \ge  (1 - \mu) f(S_1)$.
\end{lemma}

The next lemma generalizes the result obtained in~\cite{orlin2016robust,bogunovic2017robust}, and applies to any non-negative monotone set function with bipartite subadditivity ratio $\theta$.
\begin{lemma}\label{lemma:second_lower_bound}
	Let $\theta \in [0,1]$ be a bipartite subadditivity ratio defined in Eq.~\eqref{eq:2_separate_sa}. Then 
	$f(S \setminus E_S^*)$ is at least
	\begin{equation*}
	 	 \theta f(\opt_{(k - \tau, V \setminus E_S^*)})\\ - (1 - e^{-\frac{k - |S_0|}{k-\tau}})^{-1}f(S_1).
	\end{equation*}
\end{lemma}
In other words, if $f(S_1)$ is \textit{small} compared to the utility of the optimal solution, then $f(S \setminus E_S^*)$ is at least a constant factor away from the optimal solution.

\par
Next, we present our key lemma that further relates $f(S \setminus E^*_S)$ to the utility of the set $S_1$ with no deletions.

\begin{lemma}\label{lemma:S0-E0}
Let $\beta$ be a constant such that $|S_0| = \lceil \beta \tau \rceil$ and $|S_0| \leq k$, and let $\check{\nu}, \check{\alpha} \in [0,1]$ be a superadditivity ratio and generalized inverse curvature (Eq.~\eqref{eq:superadditivity} and Eq.~\eqref{eq:inverse_curvature}, respectively). Finally, let $\mu$ be a constant defined as in Lemma~\ref{lemma:mu}. Then,
\begin{equation*}
	f(S \setminus E^*_S) \geq (\beta - 1) \check{\nu} (1 - \check{\alpha}) \mu f(S_1).
\end{equation*}
\end{lemma}

\begin{proof}
We have:
	\begin{align}
		f(S \setminus E_S^*) &\geq f(S_0 \setminus E_0) \nonumber \\
		& \geq \check{\nu} \sum_{e_i \in S_0 \setminus E_0} f(\lbrace e_i \rbrace)  \label{eq:lower_bound_3_1}\\
		& \geq \frac{|S_0 \setminus E_0|}{|E_1|} \check{\nu} \sum_{e_i \in E_1} f(\lbrace e_i \rbrace)  \label{eq:lower_bound_3_2}\\
		& \geq \frac{(\beta - 1) \tau}{\tau} \check{\nu} \sum_{e_i \in E_1} f(\lbrace e_i \rbrace) \label{eq:lower_bound_3_3}\\
		& \geq (\beta - 1) \check{\nu} (1 - \check{\alpha})  \nonumber \\
		& \quad \times \sum_{i = 1}^{|E_1|} f \left(\lbrace e_i \rbrace| (S \setminus E_S^*) \cup E_1^{(i-1)} \right) \label{eq:lower_bound_3_4}\\
		& = (\beta - 1) \check{\nu} (1 - \check{\alpha}) f \left(E_1 | (S \setminus E_S^*) \right) \label{eq:lower_bound_3_5}\\
		&= (\beta - 1) \check{\nu} (1 - \check{\alpha}) \mu f(S_1) \label{eq:lower_bound_3_6}.
	\end{align}
Eq.~\eqref{eq:lower_bound_3_1} follows by the superadditivity. Eq.~\eqref{eq:lower_bound_3_2} follows from the way $S_0$ is constructed, i.e. via $\Obl$~selection that ensures $f(\lbrace i \rbrace) \geq f(\lbrace j \rbrace)$ for every $i \in S_0 \setminus E_0$ and $j \in E_1$. Eq.~\eqref{eq:lower_bound_3_3} follows from 
$	|S_0 \setminus E_0| = \lceil \beta \tau \rceil - |E_0| 
			    \geq \beta \tau - \tau 
			    = (\beta - 1) \tau$,
and $|E_1| \leq \tau$. 

To prove Eq.~\eqref{eq:lower_bound_3_4}, let $E_1 = \lbrace e_1, \cdots e_{|E_1|} \rbrace$, and let $E_1^{(i-1)} \subseteq E_1$ denote the set $\lbrace e_1, \cdots, e_{i-1} \rbrace$. Also, let $E_1^{(0)} = \emptyset$. Eq.~\eqref{eq:lower_bound_3_4} then follows from 
\[
	f(\lbrace e_i \rbrace) \geq (1 - \check{\alpha}) f \left(\lbrace e_i \rbrace| (S \setminus E_S^*) \cup E_1^{(i-1)} \right),
\]
which in turns follows from~\eqref{eq:inverse_curvature} by setting $S = \lbrace e_i \rbrace $ and $\Omega = (S \setminus E_S^*) \cup E_1^{(i-1)}$. 

Finally, Eq.~\eqref{eq:lower_bound_3_5} follows from $f \left(E_1 | (S \setminus E_S^*) \right) = \sum_{e_i \in E_1} f \left(\lbrace e_i \rbrace| (S \setminus E_S^*) \cup E_1^{(i-1)} \right)$ (telescoping sum) and Eq.~\eqref{eq:lower_bound_3_6} follows from the definition of $\mu$.
\end{proof}
\par
\textbf{Main result.} We obtain the main result by examining the maximum of the obtained lower bounds in Lemma \ref{lemma:mu}, \ref{lemma:second_lower_bound} and \ref{lemma:S0-E0}. Note, that all three obtained lower bounds depend on $f(S_1)$. In Lemma~\ref{lemma:second_lower_bound}, we benefit from $f(S_1)$ being small while the opposite is true for Lemma~\ref{lemma:mu} and~\ref{lemma:S0-E0} (both bounds are increasing in $f(S_1)$). By examining the latter two, we observe that in Lemma~\ref{lemma:mu} we benefit from $\mu$ being small (i.e.~the utility that we lose due to $E_1$ is small compared to the utility of the whole set $S_1$) while the opposite is true for Lemma~\ref{lemma:S0-E0}. By carefully balancing between these cases (see Appendix~\ref{appendix_2} for details) we arrive at our main result.
\begin{figure}
    \centering
    \includegraphics[scale=0.3]{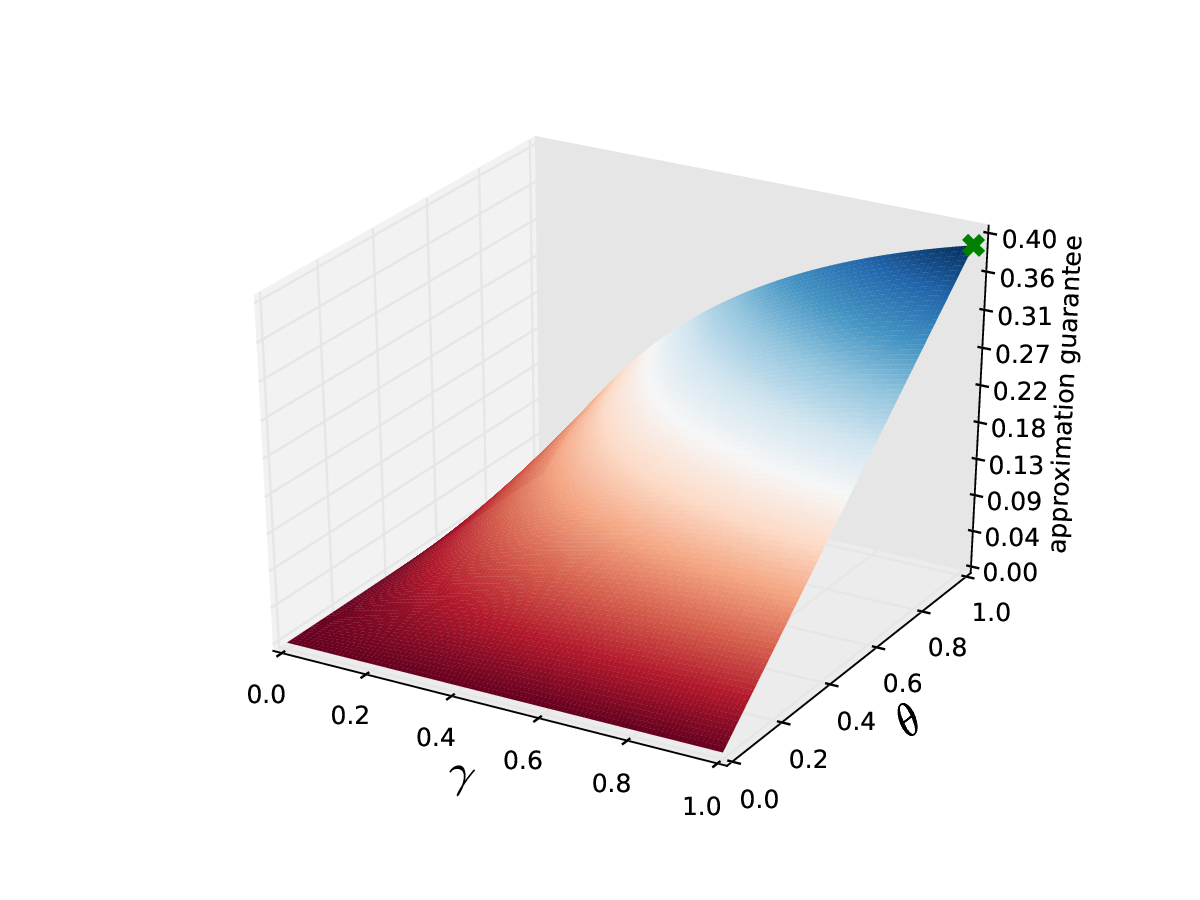}
    \caption{Approximation guarantee obtained in Remark~\ref{cor:main_remark}. The green cross represents the approximation guarantee when
    $f$ is submodular ($\gamma = \theta = 1$).
    }
    \label{fig:guarantee}
    \vspace{-2mm}
\end{figure}

\begin{theorem}\label{thm:main_thm}
Let $f:2^V \to \mathbb{R_{+}}$ be a normalized, monotone set function with submodularity ratio $\gamma$, bipartite subadditivity ratio $\theta$, inverse curvature $\check{\alpha}$ and superadditivity ratio $\check{\nu}$, every parameter in $[0,1]$. For a given budget $k$ and $\tau = \ceil{ck}$, for some $c \in (0,1)$, the $\RALG$~algorithm with $\beta $ s.t. $\ceil{\beta \tau} \leq k$ and $\beta > 1$, 
, returns a set $S$ of size $k$ such that when $k\to \infty$ we have
\[	
	f(S\setminus E_S^*) \ge \frac{\theta P \left(1-e^{-\gamma\frac{1 -   \beta c }{1- c}}\right)}{1 + P\left(1-e^{-\gamma\frac{1 -   \beta c }{1-c}}\right)} f(\opt_{(k - \tau, V \setminus E_S^*)}).
\]
where $P$ is used to denote $\frac{(\beta-1)\check{\nu}(1-\check{\alpha})}{1+(\beta-1)\check{\nu}(1-\check{\alpha})}$. 
\end{theorem}

\begin{remark}
\label{cor:main_remark}
Consider $f(\cdot)$ from Theorem~\ref{thm:main_thm} with $\check{\nu} \in (0,1]$ and $\check{\alpha} \in [0,1)$. When $\tau = o \left(\frac{k}{\beta} \right)$ and $\beta \geq \log k$ we have:
		\[f(S\setminus E_S^*)\ge \left( \theta\frac{1-e^{-\gamma}}{2-e^{-\gamma}} + o(1) \right) f(\opt_{(k - \tau, V \setminus E_S^*)}).\]
\end{remark}

\textbf{Interpretation.} 
An open question from~\cite{bogunovic2017robust} is whether a constant factor approximation guarantee is possible in the \emph{linear regime}, i.e.~when the number of removals is $\tau = \ceil{ck}$ for some constant $c \in (0,1)$ \cite{bogunovic2017robust}. In Theorem~\ref{thm:main_thm} we obtain the first asymptotic constant factor approximation in this regime.

Additionally, when $f$ is submodular, all the parameters in the obtained bound are fixed ($\check{\alpha} = 0$ and $\gamma = \theta = 1$ due to submodularity) except the superadditivity ratio $\check{\nu}$ which can take any value in $[0,1]$. The approximation factor improves for greater $\check{\nu}$, i.e.~the closer the function is to being superadditive. On the other hand, if $f$ is supermodular then $\check{\nu} = 1$ while $\check{\alpha}, \theta, \gamma$ are in $[0,1]$, and the approximation factor improves for larger $\theta$ and $\gamma$, and smaller $\check{\alpha}$.

From Remark~\ref{cor:main_remark}, when $f$ is submodular, \RALG~achieves an asymptotic approximation factor of at least $0.387$. This matches the approximation guarantee obtained in \cite{bogunovic2017robust,orlin2016robust}, while it allows for a greater number of deletions $\tau = o\left(\frac{k}{\log k}\right)$ in comparison to $\tau = o \left(\frac{k}{\log^3 k}\right)$ and $\tau= o(\sqrt{k})$ obtained in \cite{bogunovic2017robust} and \cite{orlin2016robust}, respectively. Most importantly, our result holds for a wider range of non-submodular functions. In Figure~\ref{fig:guarantee} we show how the asymptotic approximation factor changes as a function of $\gamma$ and $\theta$. 


\par
We also obtain an alternative formulation of our main result, which we present in the following corollary.

\begin{corollary}\label{cor:main_thm_cor2}
Consider the setting from Theorem~\ref{thm:main_thm} and let $P:= \frac{(\beta-1)\check{\nu}\nu}{1+(\beta-1)\check{\nu}(1-\nu)}$. Then we have 
\[	
	f(S\setminus E_S^*) \geq \frac{ \theta^2 P \left(1-e^{-\gamma\frac{1 -   \beta c }{1- c}}\right)}{1 + \theta P\left(1-e^{-\gamma\frac{1 -   \beta c }{1- c}}\right)} f(\opt_{(k - \tau, V \setminus E_S^*)}).
\]
Additionally, consider $f(\cdot)$ with $\check{\nu}, \nu \in (0,1]$. When $\tau = o \left(\frac{k}{\beta} \right)$ and $\beta \geq \log k$, as $k\to \infty$, we have that $f(S\setminus E_S^*)$ is at least
\[ \left(\frac{\theta^2 (1-e^{-\gamma})}{1 + \theta(1-e^{-\gamma})} + o(1) \right) f(\opt_{(k - \tau, V \setminus E_S^*)}).\]
\end{corollary}
The key observation is that the approximation factor depends on $\nu$ instead of inverse curvature $\check{\alpha}$. The asymptotic approximation ratio is slightly worse here compared to Theorem~\ref{thm:main_thm}. However, depending on the considered application, it might be significantly harder to provide bounds for the inverse curvature than bipartite subadditivty ratio, and hence in such cases, this formulation might be more suitable.  

\section{Applications}
\label{section_applications}
In this section, we consider two important real-world applications where deletion robust optimization is of interest. We show that the parameters used in the statement of our main theoretical result can be explicitly characterized, which implies that the obtained guarantees are applicable.

\subsection{Robust Support Selection}
We first consider the recent results that connect submodularity with concavity~\cite{elenberg2016restricted,khanna2017scalable}. 
In order to obtain bounds for robust support selection for general concave functions, we make use of the theoretical bounds obtained for \RALG~in Corollary~\ref{cor:main_thm_cor2}.

Given a differentiable concave function $l: \mathcal{X} \to \mathbb{R}$, where $\mathcal{X} \subseteq \mathbb{R}^d$ is a convex set, and $k \leq d$, the support selection problem is:
$
	\max_{\|\x \|_0 \leq k} l(\x).
$
As in~\cite{elenberg2016restricted}, we let $\text{supp}(\x) = \lbrace i : x_i \neq 0 \rbrace $, and consider the associated normalized monotone set function
\[
	f(S):=\max_{\textrm{supp}(\x)\subseteq S, \x \in \mathcal{X}} l(\x)-l(\boldsymbol{0}).
\]
Let $T_{l}(\x,\y):= l(\y) - l(\x) - \langle\nabla l(\x), \y - \x \rangle$.
An important result from \cite{elenberg2016restricted} can be rephrased as follows: if $l(\cdot)$ is $L$-smooth and $m$-strongly concave then for all $\x, \y \in \text{dom}(l)$, it holds   
\[
	- \frac{m}{2} \| \y - \x \|^2_2 \geq T_{l}(\x,\y)  \geq - \frac{L}{2} \| \y - \x \|^2_2, 
\] and $f$'s submodularity ratio $\gamma$ is lower bounded by $\frac{m}{L}$. Subsequently, in \cite{khanna2017scalable} it is shown that $\theta$ can also be lower bounded by the same ratio $\frac{m}{L}$. 

In this paper, we consider the robust support selection problem, that is, finding a set of features $S \subseteq [d]$ of size $k$ that is robust against the deletion of limited number of features. More formally, the goal is to maximize the following objective over all $S \subseteq [d]$:
\[
	\min_{|E_S| \leq \tau, E_S \subseteq S} \; \max_{\textrm{supp}(\x)\subseteq S\setminus E_S} l(\x)-l(\boldsymbol{0}).
\]

By inspecting the bound obtained in Corollary~\ref{cor:main_thm_cor2}, it remains to bound the (super/sub)additive ratio $\nu$ and $\check{\nu}$. The first bound follows by combining the result $\gamma \geq \frac{m}{L}$ with Proposition~\ref{prop:const_relation}: $\nu \geq \gamma \geq \frac{m}{L}$. To prove the second bound, we make use of the following result.
\begin{proposition}
\label{prop:support_selection}
The supermodularity ratio $\check{\gamma}$ of the considered objective $f(\cdot)$ can be lower bounded by $\frac{m}{L}$.
\end{proposition}
The second bound follows by combining the result in Proposition~\ref{prop:support_selection} and Proposition~\ref{prop:const_relation}: $\check{\nu} \geq \check{\gamma} \geq \frac{m}{L}$.

\subsection{Variance Reduction in Robust Batch Bayesian Optimization}
In batch Bayesian optimization, the goal is to optimize an unknown non-convex function from \emph{costly} concurrent function evaluations \cite{desautels2014parallelizing, gonzalez2016batch, azimi2012hybrid}. Most often, the concurrent evaluations correspond to running an expensive batch of experiments. In the case where experiments can fail, it is beneficial to select a set of experiments in a robust way.

Different acquisition (i.e.~auxiliary) functions have been proposed to evaluate the utility of candidate points for the next evaluations of the unknown function \cite{shahriari2016taking}. Recently in~\cite{bogunovic2016truncated}, the \emph{variance reduction} objective was used as the acquisition function -- the unknown function is evaluated at the points that maximally reduce variance of the posterior distribution over the given set of points that represent \emph{potential maximizers}. We formalize this as follows.

\textbf{Setup.} Let $f(\x)$ be an unknown function defined over a finite domain $\mathcal{X} = \lbrace \x_1, \cdots, \x_n \rbrace$, where $\x_i  \in \mathbb{R}^d$. Once we evaluate the function at some point $\x_i \in \mathcal{X}$, we receive a noisy observation $y_i = f(\x_i) + z$, where $z \sim \mathcal{N}(0, \sigma^2)$. In Bayesian optimization, $f$ is modeled as a sample from a Gaussian process. We use a Gaussian process with zero mean and kernel function $k(\x, \x')$, i.e. $f \sim \text{GP}(\textbf{0}, k(\x,\x'))$. Let $S = \lbrace e_1, \cdots, e_{|S|} \rbrace \subseteq [n]$ denote the set of points, and $\X_S:= [\x_{e_1}, \cdots, \x_{e_{|S|}}] \in \mathbb{R}^{|S| \times d}$ and $\y_S := [y_1, \cdots, y_{|S|}]$ denote the corresponding data matrix and observations, respectively. The posterior distribution of $f$ given the points $\X_S$ and observations $\y_S$ is again a GP, with the posterior variance given by:
\begin{align*}
	\sigma^2_{\x|S} = k(\x,\x) - k(\x,\X_{S})\left( k(\X_S, \X_S) + \sigma^2 \I_{|S|}  \right)^{-1} \\
	\times k(\X_S, \x).
\end{align*}
For a given set of potential maximizers $M \subseteq [n]$, the variance reduction objective is defined as follows:
\begin{equation}
	\label{eq:var_red}
	F_M(S) := \sum_{\x \in X_{M}} \sigma^2_\x - \sigma^2_{\x|S}, 
\end{equation} where $\sigma^2_\x = k(\x, \x)$. We show in Appendix~\ref{non-submodularity of vr} that this objective is \emph{not} submodular in general. 

Finally, our goal is to find a set of points $S$ of size $k$ that maximizes 
\[
	\min_{|E_S| \leq \tau, E_S \subseteq S} \; \sum_{\x \in X_{M}} \sigma^2_\x - \sigma^2_{\x|S\setminus E_S}.
\]
\par

In Appendix~\ref{vr}, we briefly discuss the relevant set function parameters for this objective function. We also refer the interested reader to~\cite{Halabi2019} where these are examined under further structural assumptions.

\section{Experimental Results}

\textbf{Optimization performance.}~For a returned set $S$, we measure the performance in terms of $\min_{E \subseteq S, |E| \leq \tau} f(S \setminus E)$. Note that $f(S \setminus E)$ is a submodular function in $E$. Finding the minimizer $E$ s.t. $|E| \leq \tau$ is NP-hard even to approximate~\cite{svitkina2011submodular}. We rely on the following methods in order to find $E$ of size $\tau$ that degrades the solution as much as possible:

-- Greedy adversaries: (i) \emph{Greedy Min --}~iteratively removes elements to reduce the objective value $f(S \setminus E)$ as much as possible, and (ii) \emph{Greedy Max --}~iteratively adds elements from $S$ to maximize the objective $f(E)$.

-- Random Greedy adversaries:\footnote{The random adversaries are inspired by \cite{buchbinder2014submodular} and \cite{mirzasoleiman2015}.}~In order to introduce randomness in the removal process we consider (iii) \emph{Random Greedy Min --} iteratively selects a random element from the top $\tau$ elements whose marginal gains are the highest in terms of reducing the objective value $f(S \setminus E)$ and 
(iv) \emph{Stochastic Greedy Min --} iteratively selects an element, from a random set $R \subseteq V$, with the highest marginal gain in terms of reducing $f(S \setminus E)$. At every step, $R$ is obtained by subsampling $(|S|/\tau) \log (1/\epsilon)$ elements from $S$. 

The minimum objective value $f(S \setminus E)$ among all obtained sets $E$ is reported. 
Most of the time, for all the considered algorithms, \emph{Greedy Min} finds $E$ that reduces utility the most.
\begin{figure}[h]
\vspace{-3mm}
\begin{tabular}{cccc}
\subfloat[Lin.~reg.~($\tau=10$)]{\includegraphics[width=3.8cm,height=3.2cm]{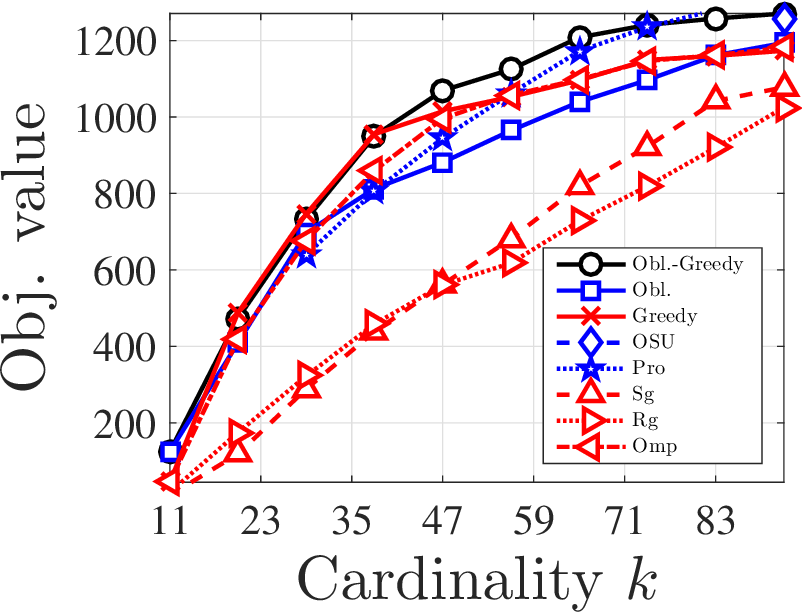}} &
\subfloat[Lin.~reg.~($\tau=30$)]{\includegraphics[width=3.8cm,height=3.2cm]{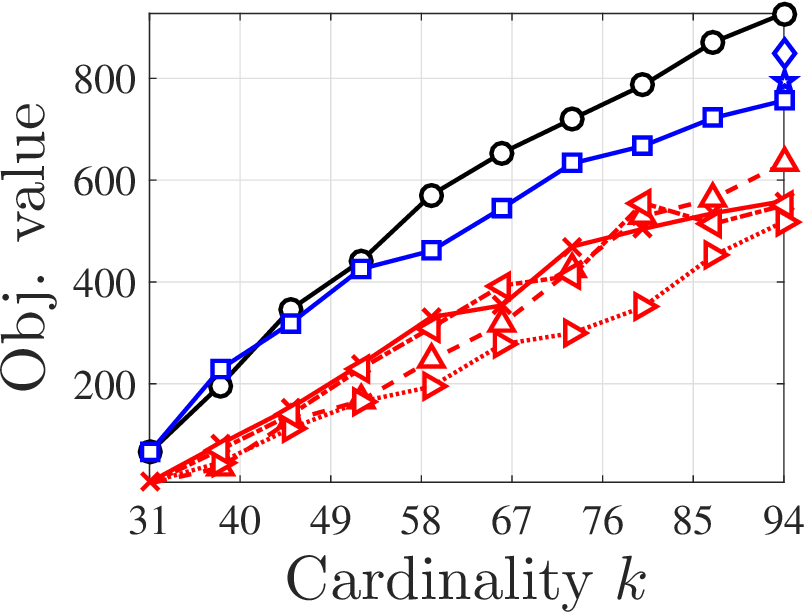}} & \\
\subfloat[Lin.~reg.~($\tau=10$)]{\includegraphics[width=3.8cm,height=3.2cm]{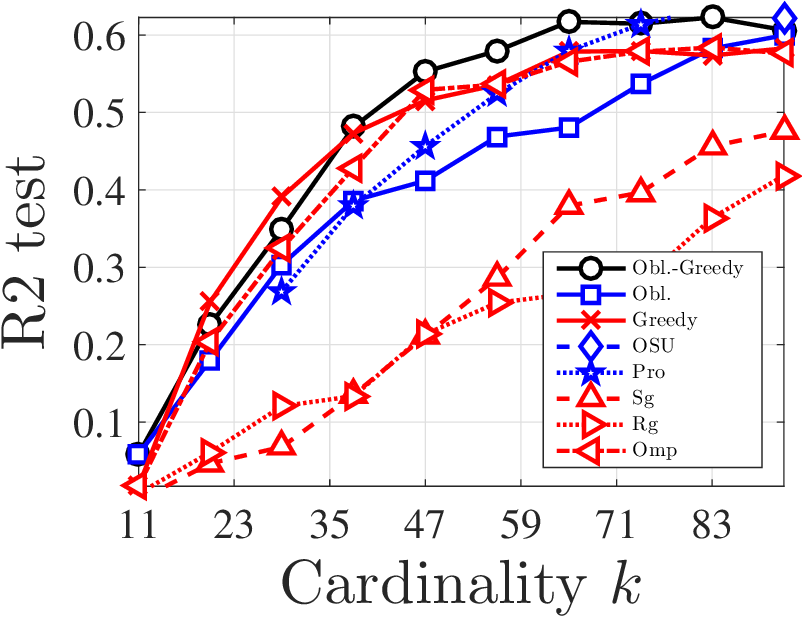}} &
\subfloat[Lin.~reg.~($\tau=30$)]{\includegraphics[width=3.8cm,height=3.2cm]{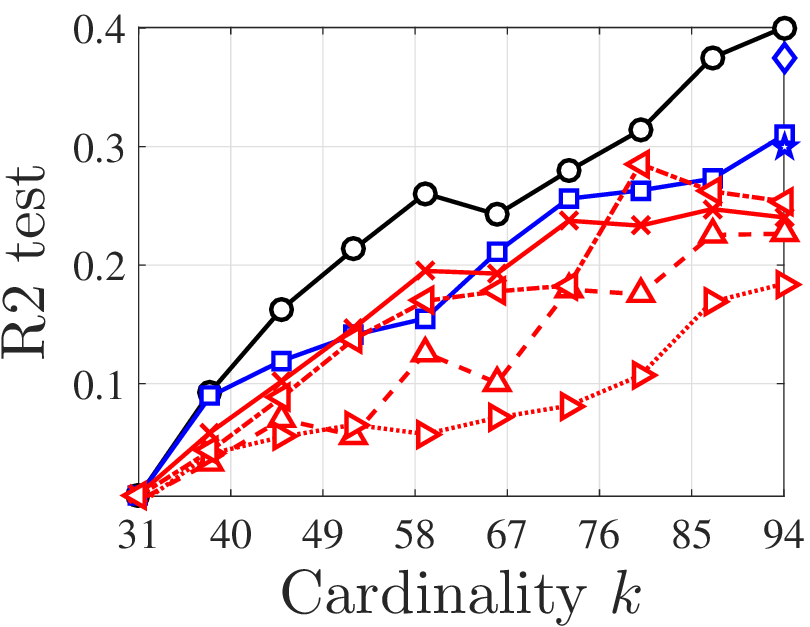}}
\end{tabular}
  \caption{Comparison of the algorithms on the linear regression task.}
\label{fig:lin_syn_fig}
\vspace{-6mm}
\end{figure}
\subsection{Robust Support Selection}
\begin{figure}[h]
\begin{tabular}{cccc}
\subfloat[Log.~synthetic ($\tau=10$)]{\includegraphics[width=3.8cm,height=3.2cm]{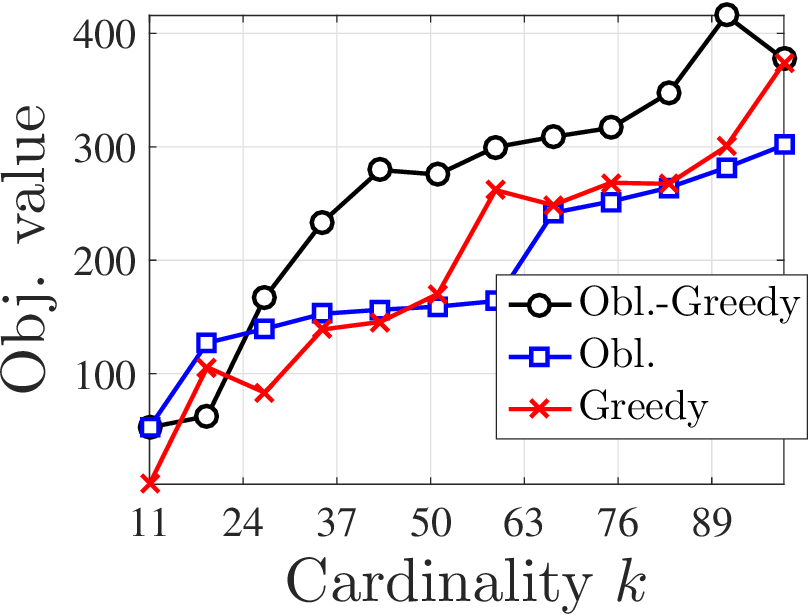}} &
\subfloat[Log.~synthetic ($\tau=30$)]{\includegraphics	[width=3.8cm,height=3.2cm]{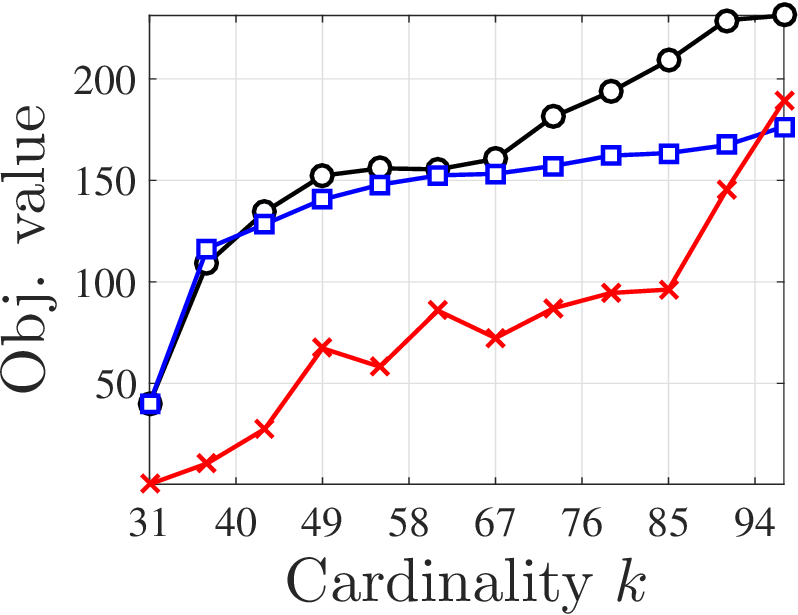}} & \\
\subfloat[Log.~synthetic ($\tau=10$)]{\includegraphics[width=3.8cm,height=3.2cm]{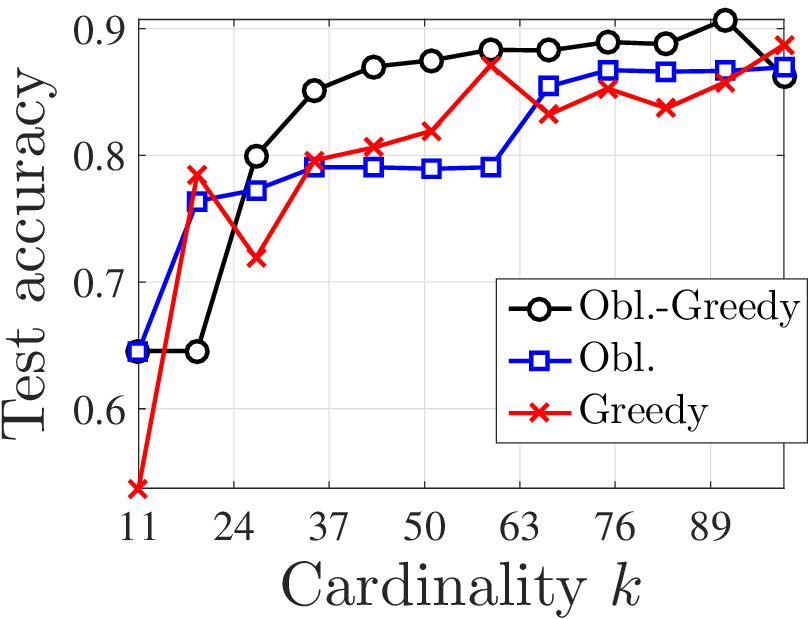}} &
\subfloat[Log.~synthetic ($\tau=30$)]{\includegraphics[width=3.8cm,height=3.2cm]{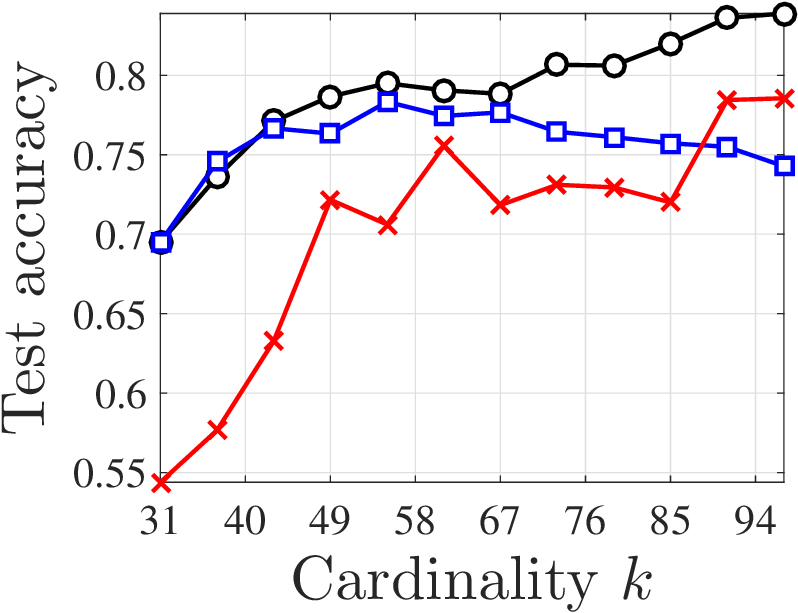}}
\end{tabular}
  \caption{Logistic regression task with synth. dataset.}
\label{fig:log_syn_fig}
\vspace{-6mm}
\end{figure}
\textbf{Linear Regression.}~Our setup is similar to the one in~\cite{khanna2017scalable}. Each row of the design matrix $\X \in \mathbb{R}^{n \times d}$ is generated by an autoregressive process,
\begin{equation}\label{eq:autoregressive}
	X_{i, t+1} = \sqrt{1 - \alpha^2} X_{i, t} + \alpha \epsilon_{i,t}, 
\end{equation}
where $\epsilon_{i,t}$ is i.i.d.~standard Gaussian with variance $\alpha^2=0.5$. We use $n=800$ training data points and $d=1000$. An additional $2400$ points are used for testing. We generate a $100$-sparse regression vector by selecting random entries of $\bomega$ and set them
$
	\bomega_s = (-1)^{\textrm{Bern}(1/2)}\times \left(5\sqrt{\frac{\log d}{n}}+\delta_s\right), 
$
where $\delta_s$ is a standard i.i.d.~Gaussian noise. The target is given by $\y=\X\bomega+\z$, where $\forall i\in [n],~z_i\sim \mathcal{N}(0, 5)$. We compare the performance of $\RALG$ against: (i) robust algorithms (in blue) such as $\Obl$, $\PRO$~\cite{bogunovic2017robust}, $\ORLIN$~\cite{orlin2016robust}, (ii) greedy-type algorithms (in red) such as $\GREEDY$, $\SGREEDY$~\cite{mirzasoleiman2015}, $\RGREEDY$~\cite{buchbinder2014submodular}, $\OMP$. We require $\beta > 1$ for our asymptotic results to hold, but we found out that in practice (small $k$ regime) $\beta \leq 1$ usually gives the best performance. We use \RALG~with $\beta=1$ unless stated otherwise.

The results are shown in Fig.~\ref{fig:lin_syn_fig}.~Since $\PRO$ and $\ORLIN$ only make sense in the regime where $\tau$ is relatively small, the plots show their performance only for feasible values of $k$. It can be observed that $\RALG$ achieves the best performance among all the methods in terms of both training error and test score. Also, the greedy-type algorithms become less robust for larger values of $\tau$.

\textbf{Logistic Regression.}~We compare the performance of $\RALG$ vs.~$\GREEDY$ and $\Obl$ selection on both synthetic and real-world data.

\textit{-- Synthetic data:} We generate a $100$-sparse $\bomega$ by letting $\bomega_s = (-1)^{\textrm{Bern}(1/2)}\times \delta_s$,
with $\delta_s\sim \textrm{Unif}([-1,1])$. The design matrix $\X$ is generated as in~\eqref{eq:autoregressive}, with $\alpha^2=0.09$.~We set $d = 200$,and use $n=600$ points for training and additional $1800$ points for testing. The label of the $i$-th data point $\X_{(i,\cdot)}$ is set to $1$ if $1/(1+\exp(\X_{(i,\cdot)}\bbeta))>0.5$ and $0$ otherwise. The results are shown in Fig.~\ref{fig:log_syn_fig}. We can observe that $\RALG$ outperforms other methods both in terms of the achieved objective value and generalization error. We also note that the performance of $\GREEDY$ decays significantly when $\tau$ increases.

\textit{-- MNIST:} We consider the $10$-class logistic regression task on the MNIST~\cite{lecun1998gradient} dataset. In this experiment, we set $\beta=0.5$ in $\RALG$, and we sample $200$ images for each digit for the training phase and $100$ images of each for testing. The results are shown in Fig.~\ref{fig:log_mnist_fig}. It can be observed that $\RALG$ has a distinctive advantage over $\GREEDY$ and $\Obl$, while when $\tau$ increases the performance of $\GREEDY$ decays significantly and more robust $\Obl$ starts to outperform it.
\begin{figure}[h]
\vspace{-3mm}
\begin{tabular}{cccc}
\subfloat[MNIST ($\tau=15$)]{\includegraphics[width=3.8cm,height=3.2cm]{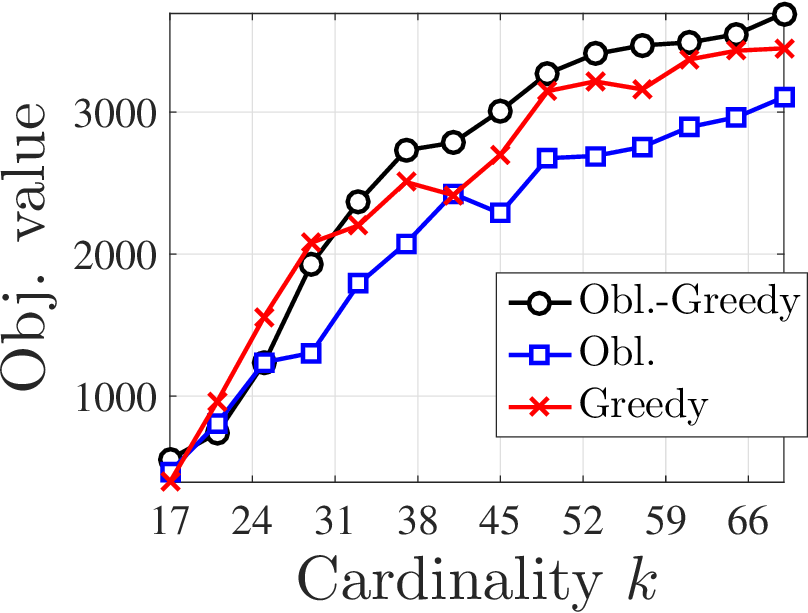}} &
\subfloat[MNIST ($\tau=45$)]{\includegraphics[width=3.8cm,height=3.2cm]{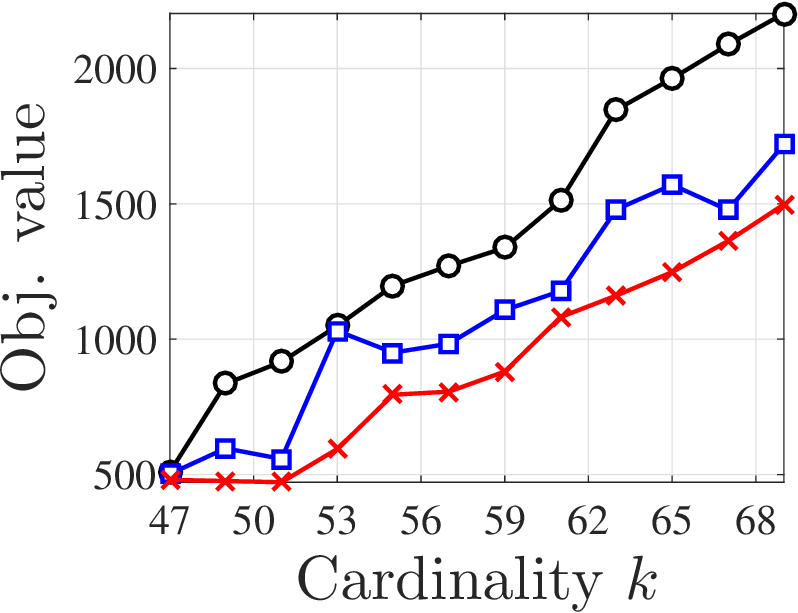}} & \\
\subfloat[MNIST ($\tau=15$)]{\includegraphics[width=3.8cm,height=3.2cm]{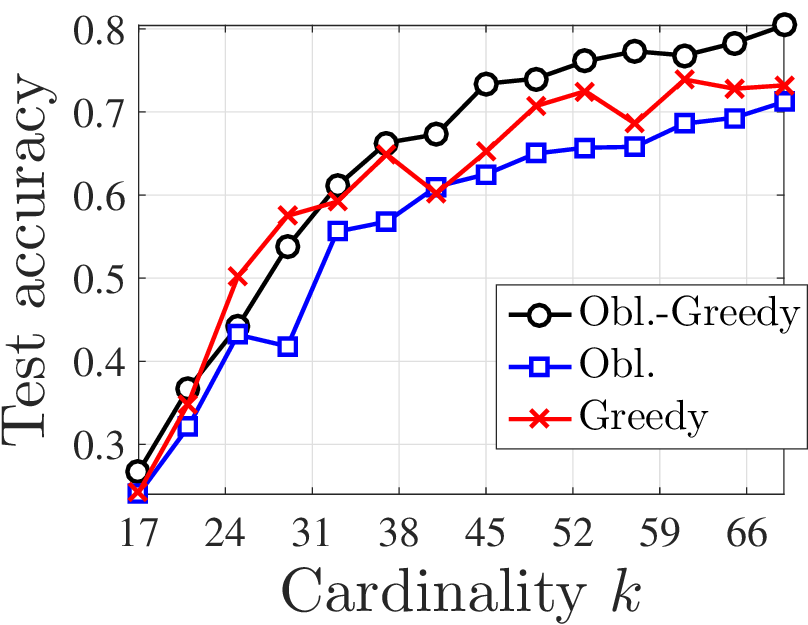}} &
\subfloat[MNIST ($\tau=45$)]{\includegraphics[width=3.8cm,height=3.2cm]{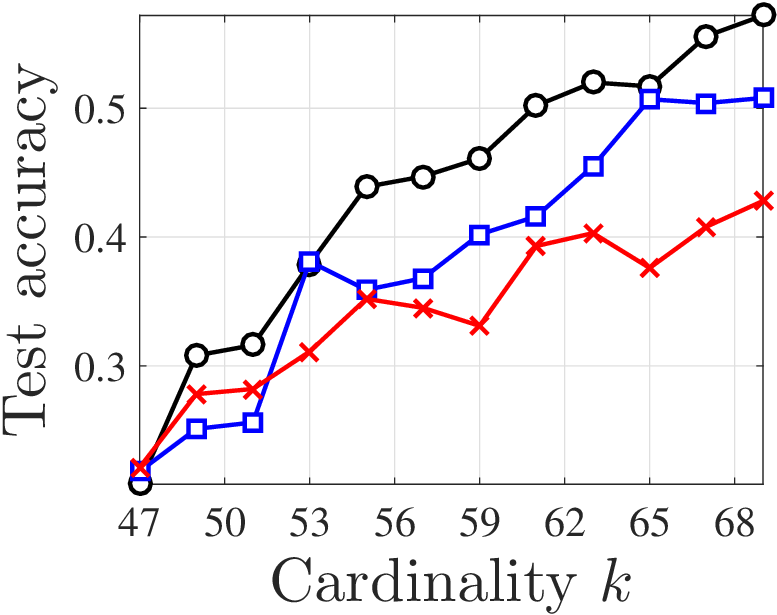}}
\end{tabular}
  \caption{Logistic regression with MNIST dataset.}
\label{fig:log_mnist_fig}
\vspace{-5mm}
\end{figure}

\subsection{Robust Batch Bayesian Optimization via Variance Reduction}
\textbf{Setup.}
We conducted the following synthetic experiment. A design matrix $X$ of size $600 \times 20$ is obtained via the autoregressive process from~\eqref{eq:autoregressive}. The function values at these points are generated from a GP with $3/2$-M\'atern kernel~\cite{rasmussen2006gaussian} with both lengthscale and output variance set to $1.0$. The samples of this function are
corrupted by Gaussian noise, $\sigma^2 = 1.0$. Objective function used is the variance reduction (Eq.~\eqref{eq:var_red}). Finally, half of the points randomly chosen are selected in the set $M$, while the other half is used in the selection process. We use $\beta = 0.5$ in our algorithm.
\begin{figure}[t!]
\vspace{-3mm}
\begin{tabular}{cccc}
\subfloat[$\alpha=0.05, \tau = 50$]{\includegraphics[width=3.8cm,height=3.2cm]{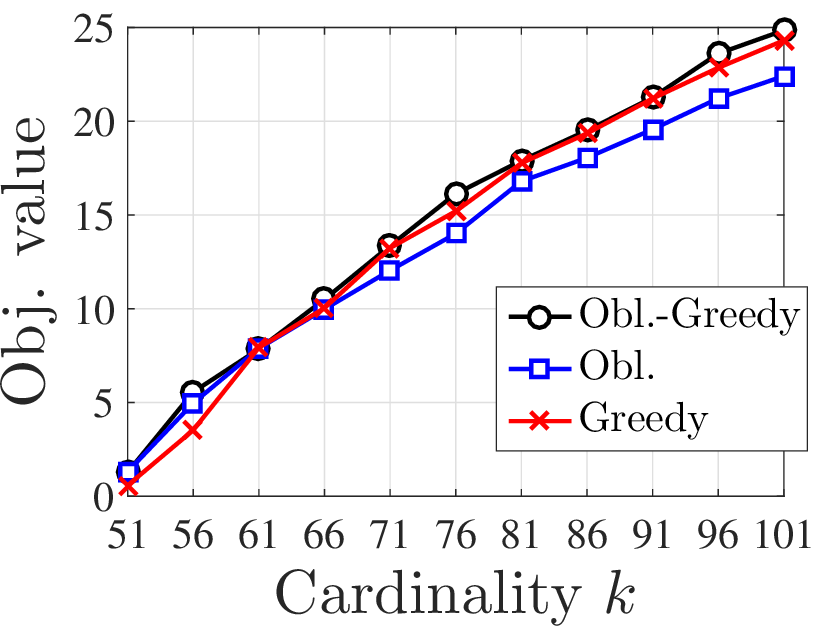}} &
\subfloat[$\alpha=0.1, \tau = 50$]{\includegraphics[width=3.8cm,height=3.2cm]{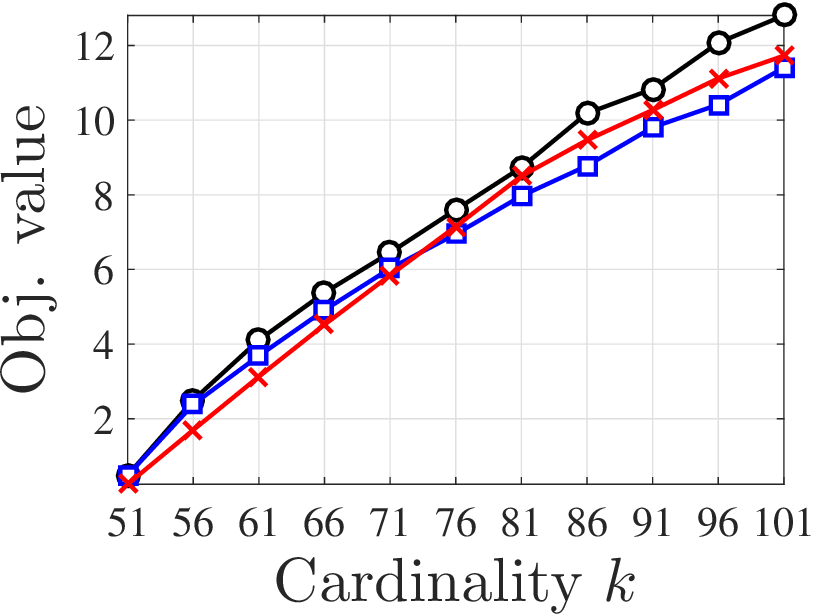}} & \\
\subfloat[$\alpha=0.2, \tau = 50$]{\includegraphics[width=3.8cm,height=3.2cm]{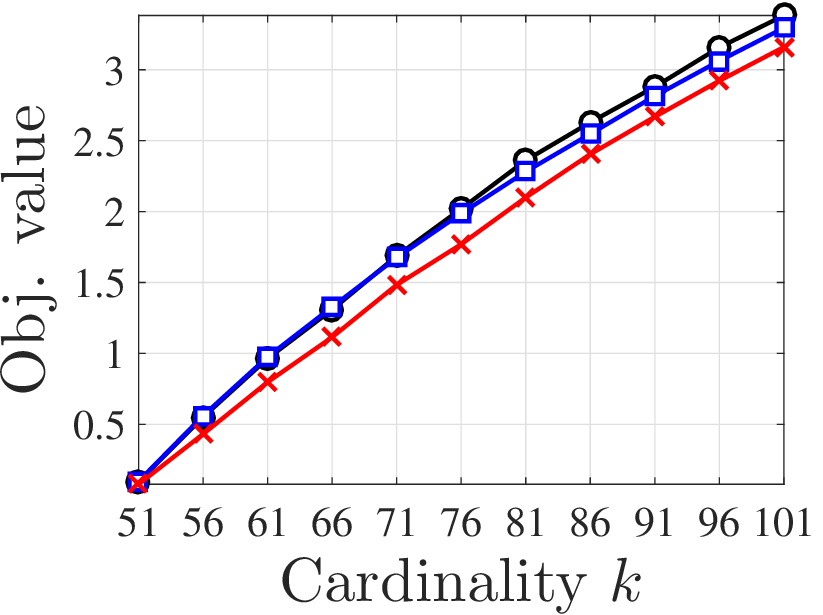}} &
\subfloat[$\alpha=0.1, k = 100$]{\includegraphics[width=3.8cm,height=3.2cm]{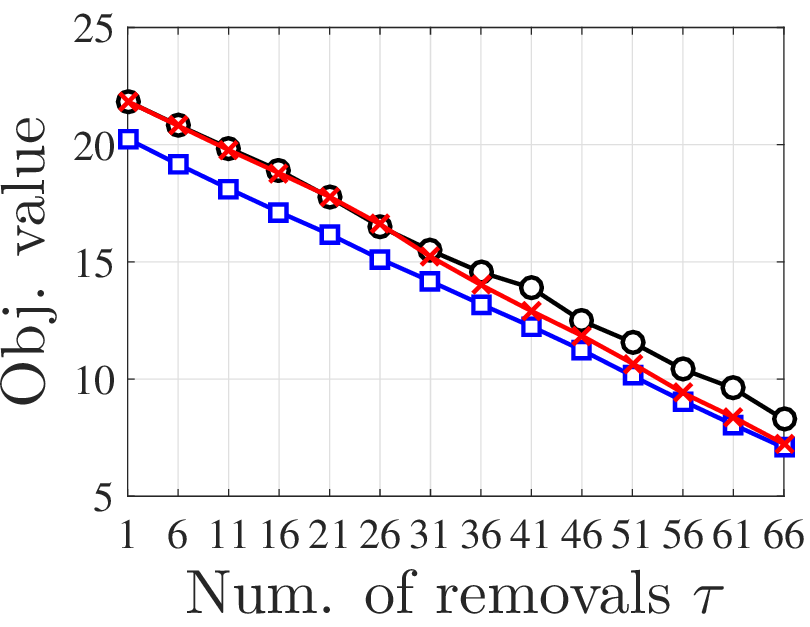}}
\end{tabular}
  \caption{Comparison of the algorithms on the variance reduction task.}
\label{fig:var_red_fig}
\vspace{-6mm}
\end{figure}

\textbf{Results.}
In Figure~\ref{fig:var_red_fig}~(a), (b), (c), the performance of all three algorithms is shown when $\tau$ is fixed to $50$. Different figures correspond to different $\alpha$ values. We observe that when $\alpha = 0.1$, \GREEDY~outperforms \Obl~for most values of $k$, while \Obl~clearly outperforms \GREEDY~when $\alpha=0.2$. For all presented values of $\alpha$, \RALG~outperforms both \GREEDY~and \Obl~selection. For larger values of $\alpha$, the correlation between the points becomes small and consequently so do the objective values. In such cases, all three algorithms perform similarly. In Figure~\ref{fig:var_red_fig}~(d), we show how the performance of all three algorithms decreases as the number of removals increases. When the number of removals is small both $\GREEDY$ and our algorithm perform similarly, while as the number of removals increases the performance of $\GREEDY$ drops more rapidly. 

\section{Conclusion}
We have presented a new algorithm \RALG~that achieves constant-factor approximation guarantees for the robust maximization of monotone non-submodular objectives. The theoretical guarantees hold for general $\tau = c k$ for some $c \in (0,1)$, which resolves the important question posed in~\cite{orlin2016robust,bogunovic2017robust}. We have also obtained the first robust guarantees for support selection and variance reduction objectives. 
In various experiments, we have demonstrated the robust performance of \RALG~by showing that it outperforms both \Obl~selection and \GREEDY, and hence achieves the best of both worlds.

\subsubsection*{Acknowledgement}
The authors would like to thank Jonathan Scarlett and Slobodan Mitrovi\'c for useful discussions. 
This work was done during JZ's summer internship at LIONS, EPFL.
IB and VC’s work was supported in part by the European Research Council (ERC) under the European Union’s Horizon 2020 research and innovation program (grant agreement
number 725594), in part by the Swiss National Science Foundation (SNF), project $407540\_167319/1$, in part by the NCCR MARVEL, funded by the Swiss National Science Foundation.

The previous version of this paper contained a proposition in Section 4.2 with bounds for inverse curvature and curvature parameters for the variance reduction objective. Its proof contained an error and is thus omitted from this paper version. We would like to thank Marwa el Halabi from MIT for pointing this out.\looseness=-1
\bibliographystyle{ieeetr}  
\bibliography{cite}

\appendix
\onecolumn
\newpage
{\Large \bf Appendix}

\begin{center}
{\large \bf Robust Maximization of Non-Submodular Objectives} \\(Ilija Bogunovic\textsuperscript{\dag}, Junyao Zhao\textsuperscript{\dag} and Volkan Cevher, AISTATS 2018)
\end{center}
\section{Organization of the Appendix}
-- Appendix~\ref{appendix_1}: Proofs from Section~\ref{set_ratios_section} \\
-- Appendix~\ref{appendix_2}: Proofs of the Main Result (Section~\ref{section_algorithm_and_guarantees}) \\
-- Appendix~\ref{appendix_3}: Proofs from Section~\ref{section_applications} \\
-- Appendix~\ref{additional_experiments}: Additional experiments

\section{Proofs from Section~\ref{set_ratios_section}}
	\label{appendix_1}
	\subsection{Proof of Proposition~\ref{prop:const_relation}}
	\begin{proof}
		We prove the following relations:
		\begin{itemize}
			\item $\boldsymbol{\nu \geq \gamma}$, $\boldsymbol{\check{\nu} \geq \check{\gamma}}$:\\
				By setting $S=\emptyset$ in both Eq.~\eqref{eq:submodularity_ratio}
				and Eq.~\eqref{eq:supermodularity_ratio}, we obtain $\forall S \subseteq V$:
				\begin{equation}
						\label{eq:nu_greater_gamma}
						\sum_{i \in S} f(\lbrace i \rbrace)	\geq \gamma f(S),
				\end{equation} and
				\begin{equation}\label{eq:c_nu_geq_c_gamma}
					f(S) \ge \check{\gamma} \sum_{i\in S} f(\{i\}).
				\end{equation}
				The result follows since, by definition of $\nu$ and $\check{\nu}$, they are the largest scalars such that Eq.~\eqref{eq:nu_greater_gamma} and Eq.~\eqref{eq:c_nu_geq_c_gamma} hold, respectively.

		\item $\boldsymbol{\gamma \geq 1 - \check{\alpha}}$, $\boldsymbol{\check{\gamma} \geq 1 - \alpha}$: \\
		Let $S, \Omega \subseteq V$ be two arbitrary disjoint sets. We arbitrarily order elements of $\Omega = \lbrace e_1,\cdots, e_{|\Omega|} \rbrace$ and we let $\Omega_{j-1}$ denote the first $j-1$ elements of $\Omega$. We also let $\Omega_0$ be an empty set. 


		By the definition of $\check{\alpha}$ (see Eq.~\eqref{eq:inverse_curvature}) we have:
		\begin{align}
			\sum_{j = 1}^{| \Omega |} f \left(\lbrace e_j  \rbrace | S \right)
			&= \sum_{ j = 1}^{| \Omega |} f \left(\lbrace e_j  \rbrace | S \cup \lbrace e_j \rbrace \setminus \lbrace e_j\rbrace \right) \nonumber  \\
			& \geq \sum_{ j = 1}^{| \Omega |} (1 - \check{\alpha}) f \left(\lbrace e_j \rbrace | S \cup \lbrace e_j \rbrace \setminus \lbrace e_j \rbrace \cup \Omega_{j-1} \right) \nonumber\\
			& =  (1 - \check{\alpha}) f \left( \Omega | S \right) \label{eq:final_alpha_check_gamma},
		\end{align}	
		where the last equality is obtained via telescoping sums. 
		
		Similarly, by the definition of $\alpha$ (see Eq.~\eqref{eq:curvature}) we have:
		\begin{align}
			(1 - \alpha) \sum_{j = 1}^{| \Omega |} f \left(\lbrace e_j  \rbrace | S \right)
			&= \sum_{ j = 1}^{| \Omega |} (1 - \alpha) f \left(\lbrace e_j  \rbrace | S \cup \lbrace e_j \rbrace \setminus \lbrace e_j\rbrace \right) \nonumber  \\
			& \le \sum_{ j = 1}^{| \Omega |} f \left(\lbrace e_j \rbrace | S \cup \lbrace e_j \rbrace \setminus \lbrace e_j \rbrace \cup \Omega_{j-1} \right) \nonumber \\
			& =  f \left( \Omega | S \right) \label{eq:final_alpha_gamma_check}.
		\end{align}

		
		
		
	  Because $S$ and $\Omega$ are arbitrary disjoint sets, and both $\gamma$ and $\check{\gamma}$ are the largest scalars such that for all disjoint sets $S,\Omega\subseteq V$ the following holds
	  $\sum_{j = 1}^{| \Omega|} f(\lbrace e_j \rbrace | S) \geq \gamma f(\Omega | S)$ and
	  $\check{\gamma} \sum_{j = 1}^{| \Omega|} f \left(\lbrace e_j  \rbrace | S \right) \le f \left( \Omega | S \right)$, it follows from Eq.~\eqref{eq:final_alpha_check_gamma} and Eq.~\eqref{eq:final_alpha_gamma_check}, respectively, that $\gamma \geq 1 - \check{\alpha}$ and $\check{\gamma} \geq 1 - \alpha$.
	\end{itemize}
	\end{proof}

	\subsection{Proof of Remark~\ref{theta_nu_inv_nu}}
	\begin{proof}
		Consider any set $S \subseteq V$, and $A$ and $B$ such that $A \cup B = S$, $A \cap B = \emptyset$.
		We have
		\[ \frac{f(A) + f(B)}{f(S)} \geq  \frac{\check{\nu} \sum_{i \in A} f(\lbrace i \rbrace) + \check{\nu} \sum_{i \in B} f(\lbrace i \rbrace)}{f(S)} = \frac{\check{\nu} \sum_{i \in S} f(\lbrace i \rbrace)}{f(S)} \geq \nu \check{\nu}, \]
where the first and second inequality follow by the definition of $\nu$ and $\check{\nu}$ (Eq.~\eqref{eq:subadditivity} and Eq.~\eqref{eq:superadditivity}), respectively.
By the definition (see Eq.~\eqref{eq:2_separate_sa}), $\theta$ is the largest scalar such that $f(A) + f(B) \geq \theta f(S)$ holds, hence, it follows $\theta \geq \nu \check{\nu}$.
	\end{proof}

\section{Proofs of the Main Result (Section~\ref{section_algorithm_and_guarantees})}
\label{appendix_2}
\subsection{Proof of Lemma~\ref{lemma:mu}}
We reproduce the proof from~\cite{bogunovic2017robust} for the sake of completeness.
\begin{proof}
\begin{align}
    f(S \setminus E_S^*)
    &= f(S) - f(S) + f(S \setminus E_S^*) \nonumber \\
    &= f(S_0 \cup S_1) + f(S\setminus E_0) - f(S \setminus E_0) - f(S)  + f(S \setminus E_S^*) \nonumber\\
    &= f(S_1) + f(S_0\ |\ S_1) + f(S\setminus E_0) - f(S) - f(S \setminus E_0) + f(S \setminus E_S^*)  \nonumber \\
    &= f(S_1) + f(S_0\ |\ (S \setminus S_0)) + f(S\setminus E_0) - f(E_0 \cup (S \setminus E_0)) - f(S \setminus E_0) + f(S \setminus E_S^*)  \nonumber \\
    &= f(S_1) + f(S_0\ |\ (S \setminus S_0)) - f(E_0\ |\ (S \setminus E_0)) - f(S\setminus E_0) + f(S \setminus E_S^*) \nonumber \\
    & =  f(S_1) + f(S_0\ |\ (S \setminus S_0)) - f(E_0\ |\ (S \setminus E_0)) - f(E_1 \cup (S \setminus E_S^*)) + f(S \setminus E_S^*)  \nonumber\\
    & =  f(S_1) + f(S_0\ |\ (S \setminus S_0)) - f(E_0\ |\ (S \setminus E_0)) - f(E_1\ |\ S \setminus E_S^*) \nonumber \\
    & =  f(S_1) - f(E_1\ |\ S \setminus E_S^*) + f(S_0\ |\ (S \setminus S_0)) - f(E_0\ |\ (S \setminus E_0)) \nonumber \\
    & \ge  (1 - \mu) f(S_1), \label{eq:fSE_pf_9} 
\end{align}
where we used $S = S_0 \cup S_1$, $E^*_S = E_0 \cup E_1$. and \eqref{eq:fSE_pf_9} follows from monotonicity, i.e., $f(S_0\ |\ (S \setminus S_0)) - f(E_0\ |\ (S \setminus E_0)) \geq 0$ (due to $E_0 \subseteq S_0$ and $S \setminus S_0 \subseteq S \setminus E_0$), along with the definition of $\mu$.
\end{proof}

\subsection{Proof of Lemma~\ref{lemma:second_lower_bound}}

\begin{proof}
We start by defining $S_0' := \opt_{(k - \tau, V \setminus E_0)} \cap (S_0 \setminus E_0)$ and $X := \opt_{(k - \tau, V \setminus E_0)} \setminus S_0'$.
\begin{align}
    f(S_0 \setminus E_0) + f(\opt_{(k - \tau, V \setminus S_0)}) 
    & \ge  f(S_0') + f(X) \label{eq:bound1_1} \\
    & \ge  \theta f(\opt_{(k - \tau, V \setminus E_0)}) \label{eq:bound1_2} \\
    & \ge  \theta f(\opt_{(k - \tau, V \setminus E_S^*)}),  
\end{align}
where \eqref{eq:bound1_1} follows from monotonicity as $S_0'  \subseteq (S_0 \setminus E_0) $ and $(V \setminus S_0) \subseteq (V \setminus E_0)$. Eq.~\eqref{eq:bound1_2}
follows from the fact that $\opt_{(k - \tau, V \setminus E_0)} = S_0' \cup X$ and the bipartite subadditive property~\eqref{eq:2_separate_sa}. The final equation follows from the definition of the optimal solution and the fact that $E^*_S = E_0 \cup E_1$.

By rearranging and noting that $f(S \setminus E^*_S) \geq f(S_0 \setminus E_0)$ due to $(S_0 \setminus E_0) \subseteq (S \setminus E^*_S)$ and monotonicity, we obtain
\[
	f(S \setminus E^*_S) \geq \theta f(\opt_{(k - \tau, V \setminus E_S^*)})  - f(\opt_{(k - \tau, V \setminus S_0)}).
\]
\end{proof}

\subsection{Proof of Theorem~\ref{thm:main_thm}}
\label{main_proof_section}
Before proving the theorem we outline the following auxiliary lemma:
\begin{lemma}[Lemma D.2 in \cite{bogunovic2017robust}]\label{lemma:max_inc_dec}
For any set function $f$, sets $A,B$, and constant $\alpha>0$, we have
\begin{equation}
\max\{\alpha f(A), \beta f(B)-f(A)\}\ge \left(\frac{\alpha}{1+\alpha}\right) \beta f(B).
\end{equation}
\end{lemma}
\par
Next, we prove the main theorem.
\begin{proof}

First we note that $\beta$ should be chosen such that the following condition holds $|S_0| = \lceil \beta \tau \rceil \leq k$. When $\tau = \ceil{ck}$ for $c \in (0,1)$ and $k \to \infty$ the condition $\beta < \frac{1}{c}$ suffices. 

We consider two cases, when $\mu=0$ and $\mu \neq 0$.
When $\mu = 0$, from Lemma~\ref{lemma:mu} we have \begin{equation}\label{eq:s_1_mu_0} f(S\setminus E_S^*) \geq f(S_1) \end{equation}
On the other hand, when $\mu\neq 0$, by Lemma \ref{lemma:mu} and \ref{lemma:S0-E0} we have 
\begin{align}
f(S\setminus E_S^*) &\ge \max\{(1-\mu)f(S_1), (\beta-1)\check{\nu}(1-\check{\alpha})\mu f(S_1)\} \nonumber \\
&\ge \frac{(\beta-1)\check{\nu}(1-\check{\alpha})}{1+(\beta-1)\check{\nu}(1-\check{\alpha})} f(S_1) \label{eq:new_mu_S1}.
\end{align}
By denoting $P := \frac{(\beta-1)\check{\nu}(1-\check{\alpha})}{1+(\beta-1)\check{\nu}(1-\check{\alpha})}$ we observe that $P \in [0,1)$ once $\beta \geq 1$. Hence, by setting $\beta \geq 1$ and taking the minimum between two bounds in Eq.~\eqref{eq:new_mu_S1} and Eq.~\eqref{eq:s_1_mu_0} we conclude that Eq.~\eqref{eq:new_mu_S1} holds for any $\mu \in [0,1]$.

By combining Eq.~\eqref{eq:new_mu_S1} with Lemma~\ref{lemma:greedy} we obtain
\begin{equation}
f(S\setminus E_S^*) \ge P \left(1-e^{-\gamma\frac{k -  \ceil{ \beta \tau} }{k-\tau}}\right)f(\opt_{(k - \tau, V \setminus S_0)}). \label{eq:final_eq}
\end{equation}

By further combining this with Lemma~\ref{lemma:second_lower_bound} we have
\begin{align}
f(S\setminus E_S^*) &\ge \max\{\theta f(\opt_{(k - \tau, V \setminus E_S^*)}) - f(\opt_{(k - \tau, V \setminus S_0)}), P \left(1-e^{-\gamma\frac{k -  \ceil{ \beta \tau} }{k-\tau}}\right)  f(\opt_{(k - \tau, V \setminus S_0)})\}\nonumber\\
&\ge \theta \frac{P \left(1-e^{-\gamma\frac{k -  \ceil{ \beta \tau} }{k-\tau}}\right)}{1 + P\left(1-e^{-\gamma\frac{k -  \ceil{ \beta \tau} }{k-\tau}}\right)} f(\opt_{(k - \tau, V \setminus E_S^*)}) \label{eq:last_mu_neq_zero_1}
\end{align}
where the second inequality follows from Lemma \ref{lemma:max_inc_dec}. By plugging in $\tau = \ceil{ck}$ we further obtain
\begin{align*}
	f(S\setminus E_S^*) &\ge \theta \frac{P \left(1-e^{-\gamma\frac{k -  \beta \ceil{ck} - 1}{(1-c)k}}\right)}{1 + P\left(1-e^{-\gamma\frac{k -  \beta\ceil{ck} - 1 }{(1-c)k}}\right)} f(\opt_{(k - \tau, V \setminus E_S^*)}) \\
	& \geq \theta \frac{P \left(1-e^{-\gamma\frac{1 -  \beta c - \frac{1}{k} - \frac{\beta}{k}}{1-c}}\right)}{1 + P\left(1-e^{-\gamma\frac{1 -  \beta c - \frac{1}{k} - \frac{\beta}{k} }{1-c}}\right)} f(\opt_{(k - \tau, V \setminus E_S^*)}) \\
	&\xrightarrow{k \to \infty} \frac{\theta P \left(1-e^{-\gamma\frac{1 -   \beta c }{1- c}}\right)}{1 + P\left(1-e^{-\gamma\frac{1 -   \beta c }{1-c}}\right)} f(\opt_{(k - \tau, V \setminus E_S^*)}).
\end{align*}

Finally, Remark~\ref{cor:main_remark} follows from Eq.~\eqref{eq:final_eq} when $\tau \in o\left(\frac{k}{\beta} \right)$ and $\beta \geq \log k$ (note that the condition $|S_0| = \lceil \beta \tau \rceil \leq k$ is thus satisfied), as $k \to \infty$, we have both $\frac{k -  \ceil{ \beta \tau} }{k-\tau} \to 1$ and $P = \frac{(\beta-1)\check{\nu}(1-\check{\alpha})}{1+(\beta-1)\check{\nu}(1-\check{\alpha})} \to 1$, when $\check{\nu} \in (0,1]$ and $\check{\alpha} \in [0,1)$.

\end{proof}

\subsection{Proof of Corollary~\ref{cor:main_thm_cor2}}
To prove this result we need the following two lemmas that can be thought of as the alternative to Lemma~\ref{lemma:mu} and~\ref{lemma:S0-E0}.
\begin{lemma}\label{lemma:alternate_mu}
Let $\mu' \in [0,1]$ be a constant such that $f(E_1) = \mu' f(S_1)$ holds. Consider $f(\cdot)$ with bipartite subadditivity ratio $\theta \in [0,1]$ defined in Eq.~\eqref{eq:two_separate_subbaditivity}. Then
\begin{equation}
	f(S \setminus E_S^*) \ge  (\theta - \mu') f(S_1). \label{eq:alternate_lower_bound_1}
\end{equation}
\end{lemma}
\begin{proof}
	By the definition of $\theta$, $f(S_1 \setminus E_1)+f(E_1) \ge \theta f(S_1)$. Hence,
	\begin{align*}
		f(S \setminus E_S^*) &\ge f(S_1 \setminus E_1) \\
		&\ge \theta f(S_1) - f(E_1)\\
		&= (\theta - \mu') f(S_1).
	\end{align*}
\end{proof}

\begin{lemma}\label{lemma:alternate_S0-E0}
Let $\beta$ be a constant such that $|S_0| = \lceil \beta \tau \rceil$ and $|S_0| \leq k$, and let $\check{\nu}, \nu \in [0,1]$ be superadditivity and subadditivity ratio (Eq.~\eqref{eq:superadditivity} and Eq.~\eqref{eq:subadditivity}, respectively). Finally, let $\mu'$ be a constant defined as in Lemma~\ref{lemma:alternate_mu}. Then,
\begin{equation}
	f(S \setminus E_S^*) \geq (\beta - 1) \check{\nu} \nu \mu' f(S_1). \label{eq:alternate_lower_bound_3}
\end{equation}
\end{lemma}

\begin{proof}
	The proof follows that of Lemma~\ref{lemma:S0-E0}, with two modifications. In Eq.~\eqref{eq:alternate_lower_bound_3_5} we used the subadditive property of $f(\cdot)$, and Eq.~\eqref{eq:alternate_lower_bound_3_6} follows by the definition of $\mu'$. 
	\begin{align}
		f(S \setminus E_S^*) &\geq f(S_0 \setminus E_0) \nonumber \\
		& \geq \check{\nu} \sum_{e_i \in S_0 \setminus E_0} f(\lbrace e_i \rbrace)  \nonumber \\
		& \geq \frac{|S_0 \setminus E_0|}{|E_1|} \check{\nu} \sum_{e_i \in E_1} f(\lbrace e_i \rbrace) \nonumber \\
		& \geq \frac{(\beta - 1) \tau}{\tau} \check{\nu} \sum_{e_i \in E_1} f(\lbrace e_i \rbrace) \nonumber \\ 
		& \geq (\beta - 1) \check{\nu} \nu f \left(E_1\right) \label{eq:alternate_lower_bound_3_5}\\
		&= (\beta - 1) \check{\nu} \nu \mu' f(S_1) \label{eq:alternate_lower_bound_3_6}.
	\end{align}
\end{proof}
\par
Next we prove the main corollary. The proof follows the steps of the proof from Appendix~\ref{main_proof_section}, except that here we make use of Lemma~\ref{lemma:alternate_mu} and~\ref{lemma:alternate_S0-E0}.
\begin{proof}

We consider two cases, when $\mu'=0$ and $\mu' \neq 0$.
When $\mu' = 0$, from Lemma~\ref{lemma:alternate_mu} we have \[f(S\setminus E_S^*) \geq \theta f(S_1).\]
On the other hand, when $\mu'\neq 0$, by Lemma \ref{lemma:alternate_mu} and \ref{lemma:alternate_S0-E0} we have 
\begin{align}
f(S\setminus E_S^*) &\ge \max\{(\theta-\mu')f(S_1), (\beta-1)\check{\nu}\nu \mu' f(S_1)\} \nonumber \\
&\ge \theta \frac{(\beta-1)\check{\nu}\nu}{1+(\beta-1)\check{\nu}\nu} f(S_1) \label{eq:mu_S1}.
\end{align}
By denoting $P := \frac{(\beta-1)\check{\nu}\nu}{1+(\beta-1)\check{\nu}\nu}$ and observing that $P \in [0,1)$ once $\beta \geq 1$, we conclude that Eq.~\eqref{eq:mu_S1} holds for any $\mu' \in [0,1]$ once $\beta \geq 1$.

By combining Eq.~\eqref{eq:mu_S1} with Lemma~\ref{lemma:greedy} we obtain
\begin{equation}
f(S\setminus E_S^*) \ge \theta P \left(1-e^{-\gamma\frac{k -  \ceil{ \beta \tau} }{k-\tau}}\right)f(\opt_{(k - \tau, V \setminus S_0)}). \label{eq:last_mu_neq_zero}
\end{equation}
By further combining this with Lemma~\ref{lemma:second_lower_bound} we have
\begin{align}
f(S\setminus E_S^*) &\ge \max\{\theta f(\opt_{(k - \tau, V \setminus E_S^*)}) - f(\opt_{(k - \tau, V \setminus S_0)}), \theta P \left(1-e^{-\gamma\frac{k -  \ceil{ \beta \tau} }{k-\tau}}\right)  f(\opt_{(k - \tau, V \setminus S_0)})\}\nonumber\\
&\ge \frac{ \theta^2 P \left(1-e^{-\gamma\frac{k -  \ceil{ \beta \tau} }{k-\tau}}\right)}{1 + \theta P\left(1-e^{-\gamma\frac{k -  \ceil{ \beta \tau} }{k-\tau}}\right)} f(\opt_{(k - \tau, V \setminus E_S^*)}), \label{eq:final_cor}
\end{align}
where the second inequality follows from Lemma \ref{lemma:max_inc_dec}. By plugging in $\tau = \ceil{ck}$ in the last equation and by letting $k \to \infty$ we arrive at:
\[f(S\setminus E_S^*) \geq \frac{ \theta^2 P \left(1-e^{-\gamma\frac{1 -   \beta c }{1- c}}\right)}{1 + \theta P\left(1-e^{-\gamma\frac{1 -   \beta c }{1- c}}\right)} f(\opt_{(k - \tau, V \setminus E_S^*)}).\]
Finally, from Eq.~\eqref{eq:final_cor}, when $\tau \in o\left(\frac{k}{\beta} \right)$ and $\beta \geq \log k$, as $k \to \infty$, we have both $\frac{k -  \ceil{ \beta \tau} }{k-\tau} \to 1$ and $P = \frac{(\beta-1)\check{\nu}\nu}{1+(\beta-1)\check{\nu}\nu} \to 1$ (when $\nu, \check{\nu} \in (0,1]$). It follows
\begin{equation*}
	f(S\setminus E_S^*) \xrightarrow{k \to \infty} \frac{\theta^2 (1-e^{-\gamma})}{1 + \theta(1-e^{-\gamma})} f(\opt_{(k - \tau, V \setminus E_S^*)}).
\end{equation*}

\end{proof}

\section{Proofs from Section~\ref{section_applications}}
\label{appendix_3}

\subsection{Proof of Proposition~\ref{prop:support_selection}}
\begin{proof}
The goal is to prove: $\check{\gamma} \ge \frac{m}{L}$.

Let $S \subseteq [d]$ and $\Omega \subseteq [d]$ be any two disjoint sets, and for any set $A \subseteq [d]$ let $\x^{(A)}=\argmax_{\textrm{supp}(\x)\subseteq A, \x \in \mathcal{X}} l(\x)$. Moreover, for $B \subseteq [d]$ let $\x^{(A)}_B$ denote those coordinates of vector $\x^{(A)}$ that correspond to the indices in $B$.

We proceed by upper bounding the denominator and lower bounding the numerator in~\eqref{eq:supermodularity_ratio}. By definition of $\x^{(S)}$ and strong concavity of $l(\cdot)$,

\begin{align*}
	l(\x^{(S\cup\{i\})})-l(\x^{(S)}) &\le \langle \nabla l(\x^{(S)}), \x^{(S\cup\{i\})} -\x^{(S)}\rangle - \frac{m}{2}\norm{\x^{(S\cup\{i\})} - \x^{(S)}}^2\\
	& \le \max_{\v:\v_{(S \cup \{i\})^c=0}} \langle \nabla l(\x^{(S)}), \v-\x^{(S)}\rangle - \frac{m}{2}\norm{\v-\x^{(S)}}^2 \\
	& = \frac{1}{2m}\norm{\nabla l(\x^{(S)})_{i}}^2
\end{align*}

where the last equality follows by plugging in the maximizer $\v=\x^{(S)}+\frac{1}{m}\nabla l(\x^{(S)})_{i}$. Hence,
\[
	\sum_{i\in\Omega} \left(l(\x^{(S\cup\{i\})})-l(\x^{(S)})\right) \le \sum_{i\in\Omega} \frac{1}{2m}\norm{\nabla l(\x^{(S)})_{i}}^2 = \frac{1}{2m}\norm{\nabla l(\x^{(S)})_{\Omega}}^2.
\]
On the other hand, from the definition of $\x^{(S\cup \Omega)}$ and due to smoothness of $l(\cdot)$ we have
\begin{equation*}
\begin{split}
l(\x^{(S\cup \Omega)}) - l(\x^{(S)}) &\ge l(\x^{(S)}+\frac{1}{L}\nabla l(\x^{(S)})_{\Omega}) - l(\x^{(S)})\\
&\ge \langle\nabla l(\x^{(S)}), \frac{1}{L}\nabla l(\x^{(S)})_{\Omega}\rangle - \frac{L}{2}\norm{\frac{1}{L}\nabla l(\x^{(S)})_{\Omega}}^2 \\
&= \frac{1}{2L}\norm{l(\x^{(S)})_{\Omega}}^2.
\end{split}
\end{equation*}
It follows that
\[
	\frac{l(\x^{(S\cup \Omega)}) - l(\x^{(S)})}{\sum_{i\in\Omega} \left(l(\x^{(S\cup\{i\})})-l(\x^{(S)})\right)} \ge \frac{m}{L},\quad \forall \textrm{ disjoint } S,\Omega\subseteq [d] 
\]
We finish the proof by noting that $\check{\gamma}$ is the largest constant for the above statement to hold.

\end{proof}

\subsection{Variance Reduction in GPs}
\label{vr}

\subsubsection{Non-submodularity of Variance Reduction} \label{non-submodularity of vr}
The goal of this section is to show that the GP variance reduction objective is not submodular in general.
Consider the following PSD kernel matrix:
\[
\K =
\begin{bmatrix}
    1 & \sqrt{1 - z^2} & 0 \\
    \sqrt{1 - z^2} & 1 & z^2 \\
    0 & z^2 & 1  \\
\end{bmatrix}.
\]

We consider a single $x = \lbrace 3 \rbrace$ (i.e. $M$ is a singleton) that corresponds to the third data point. The objective is as follows:  
\[F(i|S) = \sigma^2_{\lbrace 3 \rbrace|S} - \sigma^2_{\lbrace 3 \rbrace|S \cup i}.\]

The submodular property implies $F(\lbrace 1 \rbrace) \geq F(\lbrace 1 \rbrace| \lbrace 2 \rbrace)$.
We have:
\begin{align*}
	F(\lbrace 1 \rbrace) &= \sigma^2_{\lbrace 3 \rbrace} - \sigma^2_{\lbrace 3  \rbrace | \lbrace 1 \rbrace} \\
	&= 1  - K(\lbrace 3 \rbrace, \lbrace 3\rbrace) - K(\lbrace 3 \rbrace, \lbrace 1\rbrace) (K(\lbrace 1 \rbrace, \lbrace 1\rbrace) + \sigma^2)^{-1} K(\lbrace 1 \rbrace, \lbrace 3 \rbrace)\\ 
	&= 1 - 1 + 0 = 0,
\end{align*}
and
\begin{align*}
	F(\lbrace 2 \rbrace) &= \sigma^2_{\lbrace 3 \rbrace} - \sigma^2_{\lbrace 3  \rbrace | \lbrace 2 \rbrace} \\
	&= 1  - K(\lbrace 3 \rbrace, \lbrace 3\rbrace) - K(\lbrace 3 \rbrace, \lbrace 2\rbrace) (K(\lbrace 2 \rbrace, \lbrace 2\rbrace) + \sigma^2)^{-1} K(\lbrace 2 \rbrace, \lbrace 3 \rbrace)\\ 
	&= 1 - (1 - z^2 (1 + \sigma^2)^{-1} z^2) = \frac{z^4}{1 + \sigma^2},
\end{align*}
and
\begin{align*}
	F(\lbrace 1, 2 \rbrace) &= \sigma^2_{\lbrace 3 \rbrace} - \sigma^2_{\lbrace 3  \rbrace | \lbrace 1,2 \rbrace} \\
	&= 1  - K(\lbrace 3 \rbrace, \lbrace 3\rbrace) + [K(\lbrace 3 \rbrace, \lbrace 1\rbrace), K(\lbrace 3 \rbrace, \lbrace 2\rbrace)] 
	\begin{bmatrix}
           1 + \sigma^2,  K(\lbrace 2 \rbrace, \lbrace 1 \rbrace)\\
           K(\lbrace 1 \rbrace, \lbrace 2 \rbrace), 1 + \sigma^2 
    \end{bmatrix}^{-1} 
	\begin{bmatrix}
           K(\lbrace 1 \rbrace, \lbrace 3 \rbrace) \\
           K(\lbrace 2 \rbrace, \lbrace 3 \rbrace)
    \end{bmatrix} \\ 
	&= 1  - 1 + [0, z^2] 
	\begin{bmatrix}
           1 + \sigma^2,  \sqrt{1 - z^2}\\
           \sqrt{1 - z^2}, 1 + \sigma^2 
    \end{bmatrix}^{-1} 
	\begin{bmatrix}
           0 \\
           z^2
    \end{bmatrix}\\
    &= \frac{z^4(1 + \sigma^2)}{(1 + \sigma^2)^2 - (1 - z^2)}.
\end{align*}

We obtain,
\begin{align*}
	F(\lbrace 1 \rbrace| \lbrace 2 \rbrace) &= F(\lbrace 1, 2 \rbrace) - F(\lbrace 2 \rbrace)\nonumber \\
	&= \frac{z^4}{(1 + \sigma^2) - (1 - z^2)(1 + \sigma^2)^{-1}} -  \frac{z^4}{1 + \sigma^2} \label{eq:F_12}.
\end{align*}
When $z \in (0,1)$, $F(\lbrace 1 \rbrace| \lbrace 2 \rbrace)$ is strictly greater than $0$, and hence greater than $F(\lbrace 1 \rbrace)$. This is in contradiction with the submodular property which implies $F(\lbrace 1 \rbrace) \geq F(\lbrace 1 \rbrace| \lbrace 2 \rbrace)$.

\subsubsection{Variance Reduction Curvature}
\label{prop:prop_vr_proof}
	
	In this section, we are interested in lower bounding the following ratio: $\frac{f( \lbrace i \rbrace | S \setminus \lbrace i \rbrace \cup \Omega)}{f(\lbrace i \rbrace| S \setminus \lbrace i \rbrace)}$.

	Let $k_{\text{max}} \in \mathbb{R_{+}}$ be the largest variance, i.e. $k(\x_i, \x_i) \leq k_{\text{max}}$ for every $i$. Consider the case when $M$ is a singleton set:
	\[
		f(i|S) = \sigma^2_{\x|S} - \sigma^2_{\x|S \cup i}.
	\] By using $\Omega = \lbrace i \rbrace$ in Eq.~\eqref{eq:vr_alternative_form}, we can rewrite $f(i|S)$ as
	\[
		f(i|S) = a_{i,S}^2 B_i^{-1},	
	\] where $a_{i,S}, B_i \in \mathbb{R_{+}}$, and are given by:
	\[
		a_{i,S} = k(\x, \x_i) - k(\x,\X_S)(k(\X_S, \X_S) + \sigma^2 \I)^{-1} k(\X_S, \x_i)
	\] and 
	\[
		B_i = \sigma^2 + k(\x_i, \x_i) - k(\x_i,\X_S)(k(\X_S,\X_S) + \sigma^2 \I)^{-1}k(\X_S, \x_i).
	\]
	By using the fact that $k(\x_i, \x_i) \leq k_{\text{max}}$, for every $i$ and $S$, we can upper bound $B_i$ by $\sigma^2 + k_{\text{max}}$ (note that $k(\x_i, \x_i) - k(\x_i,\X_S)(k(\X_S,\X_S) + \sigma^2 \I)^{-1}k(\X_S, \x_i) \geq 0$ as variance cannot be negative), and lower bound by $\sigma^2$. It follows that for every $i$ and $S$ we have:
	\[
		\frac{a_{i,S}^2}{\sigma^2 + k_{\text{max}}} \leq f(i|S) \leq \frac{a_{i,S}^2}{\sigma^2}. 
	\]
	Therefore,
	\begin{align*}
		\frac{f( \lbrace i \rbrace | S \setminus \lbrace i \rbrace \cup \Omega)}{f(\lbrace i \rbrace| S \setminus \lbrace i \rbrace)} &\geq 
		\frac{\sigma^2}{\sigma^2+k_{\text{max}}}	
		\frac{a_{i, S \setminus \lbrace i \rbrace \cup \Omega}^2}{a_{i, S \setminus \lbrace i \rbrace}^2}, \quad \forall S, \Omega \subseteq V, i \in S \setminus \Omega.	
	\end{align*}
	Hence, the curvature of the variance reduction objective depends on the following ratio $\frac{a_{i, S \setminus \lbrace i \rbrace \cup \Omega}^2}{a_{i, S \setminus \lbrace i \rbrace}^2}$. Under some further structural assumption this ratio can be bounded. We refer the interested reader to \cite{Halabi2019} for further details.  

\subsubsection{Alternative GP variance reduction form}
Here, the goal is to show that the variance reduction can be written as
\begin{equation}
\label{eq:vr_alternative_form}
	F(\Omega|S) = \sigma^2_{\x|S} - \sigma^2_{\x|S\cup \Omega} =  \a \B^{-1} \a^T,
\end{equation}
where $\a \in \mathbb{R}_{+}^{1 \times |\Omega \setminus S|}$, $\B \in \mathbb{R}_{+}^{|\Omega \setminus S| \times |\Omega \setminus S|}$ and are given by:
\[
	\a :=  k(\x, \X_{\Omega \setminus S}) - k(\x,\X_S)(k(\X_S, \X_S) + \sigma^2 \I)^{-1} k(\X_S, \X_{\Omega \setminus S}),
\] and
\[
	\B := \sigma^2 \I + k(\X_{\Omega \setminus S}, \X_{\Omega \setminus S}) - k(\X_{\Omega \setminus S},\X_S)(k(\X_S,\X_S) + \sigma^2 \I)^{-1}k(\X_S, \X_{\Omega \setminus S}).
\]

This form is used in the proof in Appendix~\ref{prop:prop_vr_proof}.
\begin{proof}
Recall the definition of the posterior variance:
\[
 	\sigma^2_{\x|S} = k(\x,\x) - k(\x,\X_{S})\left( k(\X_S, \X_S) + \sigma^2 \I_{|S|}  \right)^{-1} k(\X_S, \x).
\]

We have
\begin{align*}
 	F(\Omega|S) &= \sigma^2_{\x|S} - \sigma^2_{\x|S\cup \Omega}  \\
 				&= k(\x,\X_{S \cup \Omega})\left( k(\X_{S \cup \Omega}, \X_{S \cup \Omega}) + \sigma^2 \I_{|\Omega \cup S|}  \right)^{-1} k(\X_{S \cup \Omega}, \x) - k(\x,\X_{S})\left( k(\X_S, \X_S) + \sigma^2 \I_{|S|}  \right)^{-1} k(\X_S, \x)\\
 				&= [\m_1, \m_2] 
 				\begin{bmatrix}
            			\A_{11},  \A_{12}\\
            			\A_{21},  \A_{22} 
     			\end{bmatrix}^{-1}
     			\begin{bmatrix}
            			\m_1^T \\
            			\m_2^T
     			\end{bmatrix} 
     			- \m_1 \A_{11}^{-1} \m_1^T,
\end{align*}
where we use the following notation:
\begin{align*}
	\m_1 &:= k(\x, \X_S),\\
	\m_2  &:= k(\x, \X_{\Omega \setminus S}),\\
	\A_{11} &:= k(\X_S, \X_S) + \sigma^2 \I_{|S|},\\
	\A_{12} &:= k(\X_S, \X_{\Omega \setminus S}), \\
	\A_{21} &:= k(\X_{\Omega \setminus S}, X_S), \\
	\A_{22} &:= k(\X_{\Omega \setminus S}, \X_{\Omega \setminus S}) + \sigma^2 \I_{|\Omega \setminus S|}. 
\end{align*}

By using the inverse formula~\cite[Section 9.1.3]{petersen2008matrix} we obtain:
\begin{align*}
	F(\Omega|S) &= [\m_1, \m_2] 
				\begin{bmatrix*}
           			\A_{11}^{-1} + \A_{11}^{-1} \A_{12} \B^{-1} \A_{21} \A_{11}^{-1},&  -\A_{11}^{-1} \A_{12} \B^{-1}\\
           			-\B^{-1} \A_{21} \A_{11}^{-1},&  \B^{-1} 
    			\end{bmatrix*}
    			\begin{bmatrix}
           			\m_1^T \\
           			\m_2^T
    			\end{bmatrix} 
    			- \m_1 \A_{11}^{-1} \m_1^T,
\end{align*}
where 
\[
	\B := \A_{22} - \A_{21}\A_{11}^{-1} \A_{12}.
\]
\par
Finally, we obtain:
\begin{align*}
	F(\Omega|S) &= \m_1 \A_{11}^{-1} \m_1^T + \m_{1} \A_{11}^{-1} \A_{12} \B^{-1} \A_{21}\A_{11}^{-1}\m_1^T - \m_2 \B^{-1} \A_{21}\A_{11}^{-1}\m_1^T\\ & ~- \m_1 \A_{11}^{-1} \A_{12} \B^{-1}\m_{2}^T + \m_2 \B^{-1} \m_2^T - \m_1 \A_{11}^{-1} \m_1^T \\
	&= \m_{1} \A_{11}^{-1} \A_{12} \B^{-1} (\A_{21}\A_{11}^{-1}\m_1^T - \m_2^{T}) - \m_2 \B^{-1} (\A_{21}\A_{11}^{-1}\m_1^T - \m_2^{T}) \\
	&= (\m_{1} \A_{11}^{-1} \A_{12} - \m_2) \B^{-1} (\A_{21}\A_{11}^{-1}\m_1^T - \m_2^{T}) \\
	&= (\m_2 - \m_{1} \A_{11}^{-1} \A_{12}) \B^{-1} (\m_2^{T} - \A_{21}\A_{11}^{-1}\m_1^T).
\end{align*}
By setting 
\begin{align*}
	\a &:= \m_2 - \m_{1} \A_{11}^{-1} \A_{12} \\
	  &= k(\x, \X_{\Omega \setminus S}) - k(\x,\X_S)(k(\X_S, \X_S) + \sigma^2 \I)^{-1} k(\X_S, \X_{\Omega \setminus S})
\end{align*}
and 
\begin{align*}
	\a^T &:= \m_2^{T} - \A_{21}\A_{11}^{-1}\m_1^T \\
	    &= k(\X_{\Omega \setminus S}, \x) - k(\X_{\Omega \setminus S}, \X_S)(k(\X_S, \X_S) + \sigma^2 \I)^{-1}  k(\X_S, \x), 
\end{align*}
we have
\[
	F(\Omega|S) = \a \B^{-1} \a^T,
\] where
\[
	\B = \sigma^2 \I_{|\Omega \setminus S|} + k(\X_{\Omega \setminus S}, \X_{\Omega \setminus S}) - k(\X_{\Omega \setminus S}, \X_S)(k(\X_S,\X_S) + \sigma^2 \I_{|S|})^{-1}k(\X_S, \X_{\Omega \setminus S}).
\]
\end{proof}
\newpage
\section{Additional Experiments}
\label{additional_experiments}
\begin{figure}[h]\label{fig:supp_additional}
\begin{tabular}{cccc}
\subfloat[Lin.~reg.~($\tau=20$)]{\includegraphics[width=3.8cm,height=3.2cm]{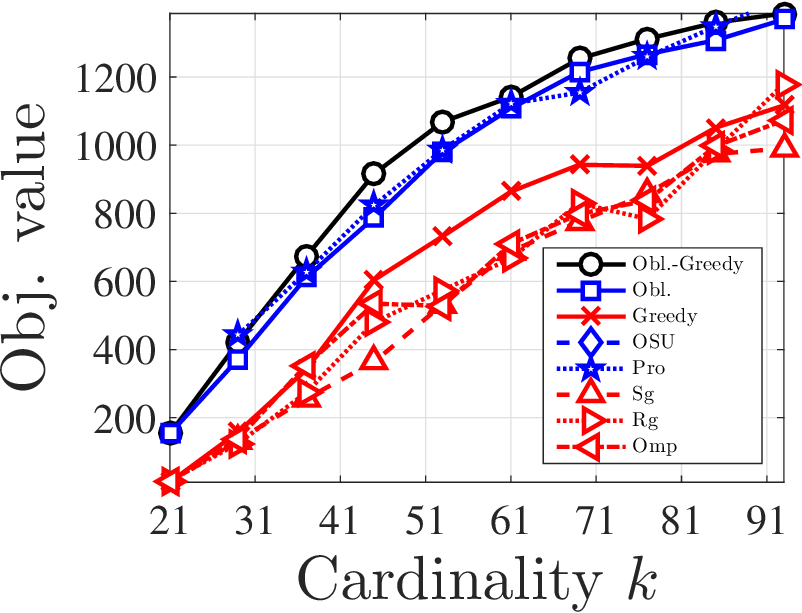}} &
\subfloat[Lin.~reg.~($\tau=40$)]{\includegraphics[width=3.8cm,height=3.2cm]{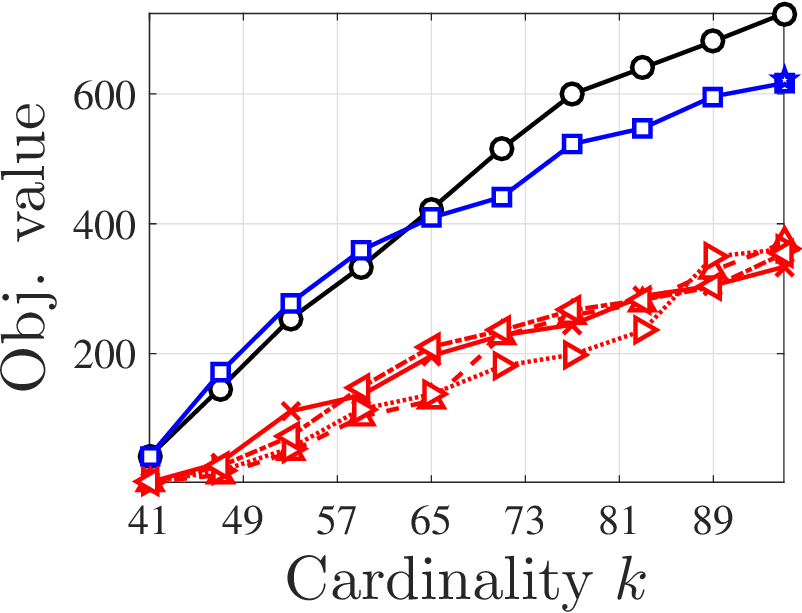}} &
\subfloat[Lin.~reg.~($\tau=20$)]{\includegraphics[width=3.8cm,height=3.2cm]{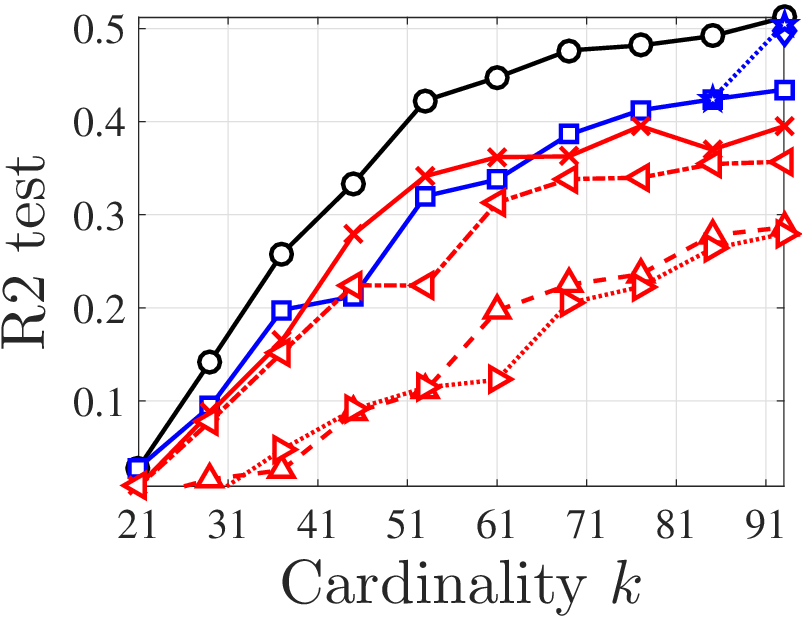}} &
\subfloat[Lin.~reg.~($\tau=40$)]{\includegraphics[width=3.8cm,height=3.2cm]{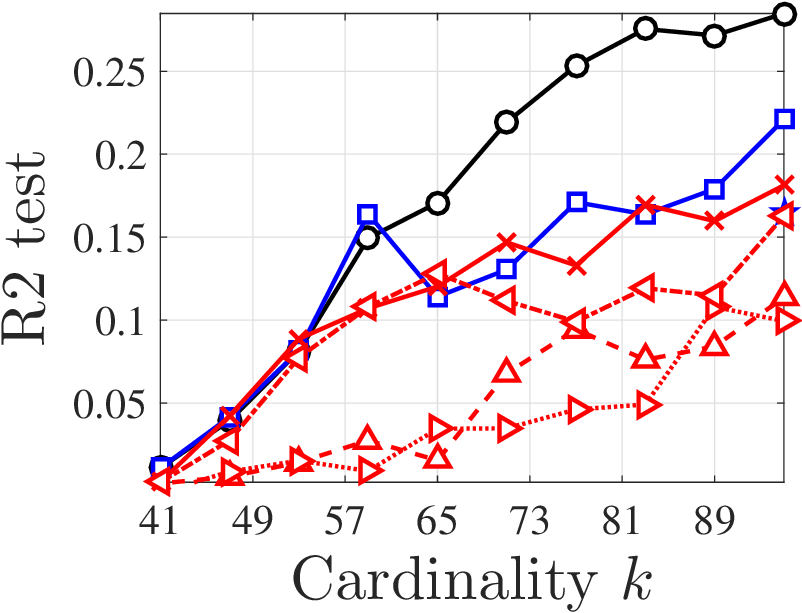}} \\
\subfloat[Log.~synthetic ($\tau=20$)]{\includegraphics[width=3.8cm,height=3.2cm]{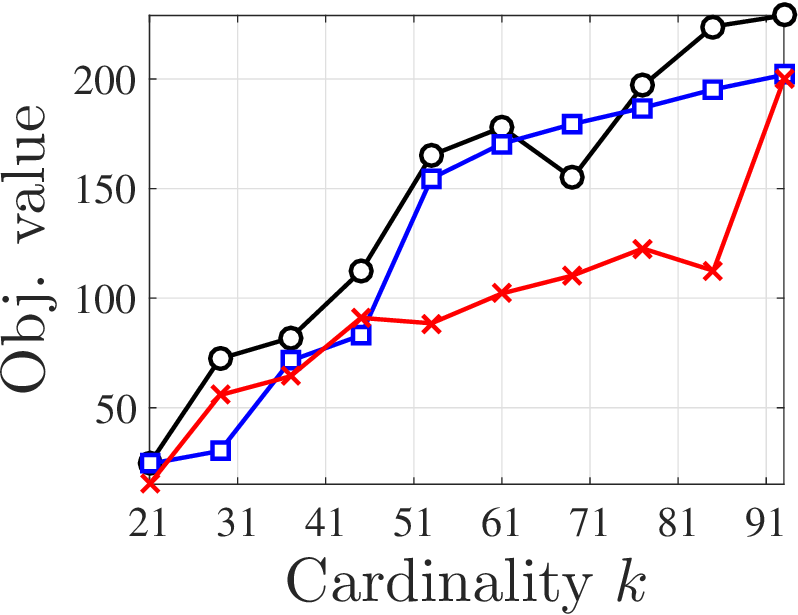}} &
\subfloat[Log.~synthetic ($\tau=40$)]{\includegraphics	[width=3.8cm,height=3.2cm]{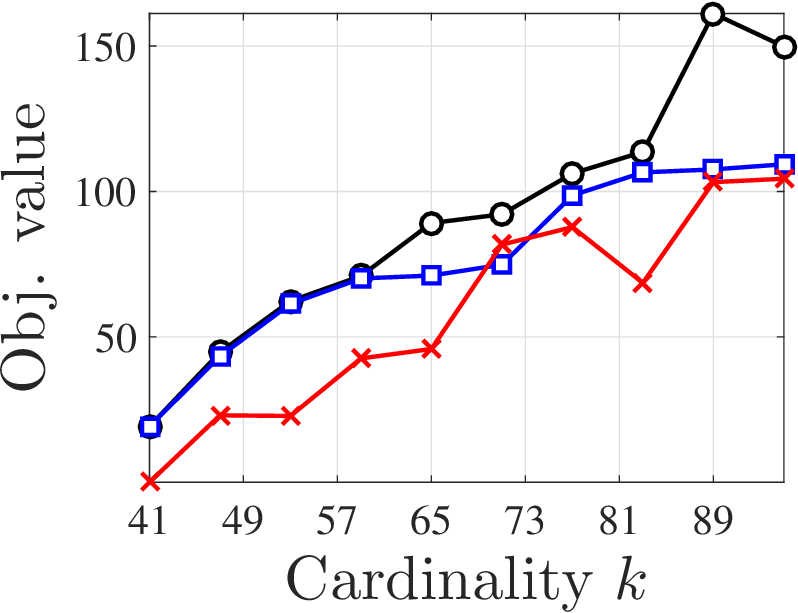}} &
\subfloat[Log.~synthetic ($\tau=20$)]{\includegraphics[width=3.8cm,height=3.2cm]{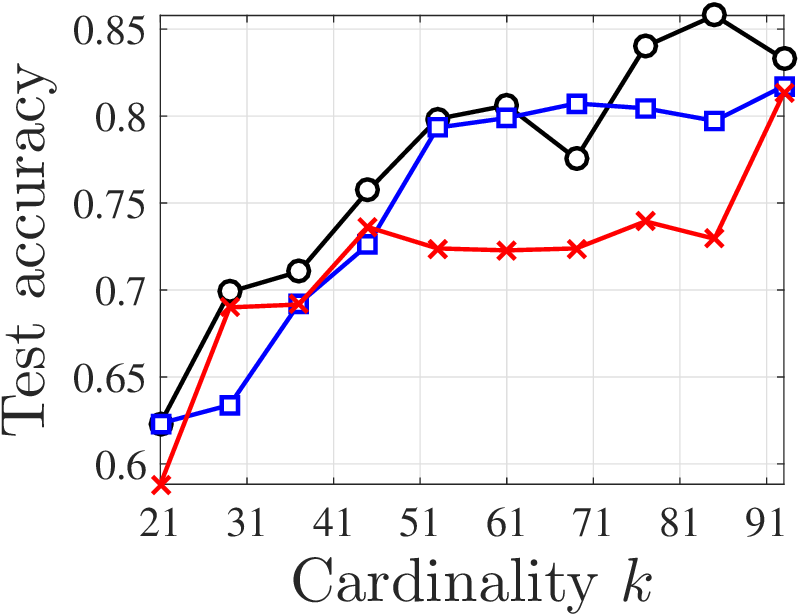}} &
\subfloat[Log.~synthetic ($\tau=40$)]{\includegraphics[width=3.8cm,height=3.2cm]{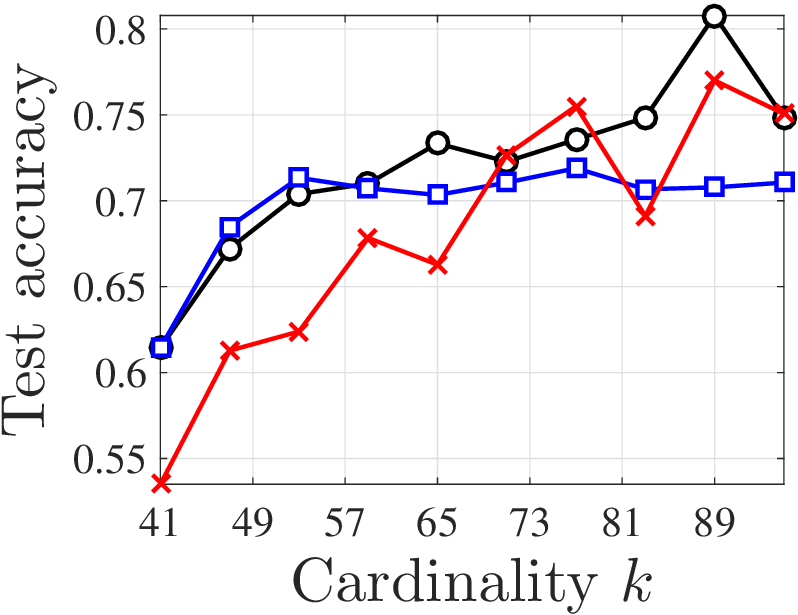}} \\
\subfloat[MNIST ($\tau=25$)]{\includegraphics[width=3.8cm,height=3.2cm]{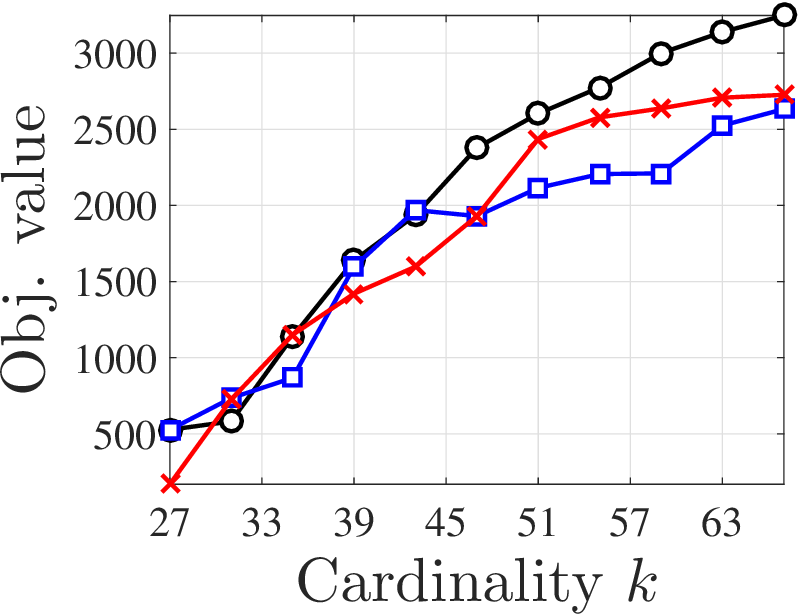}} &
\subfloat[MNIST ($\tau=35$)]{\includegraphics[width=3.8cm,height=3.2cm]{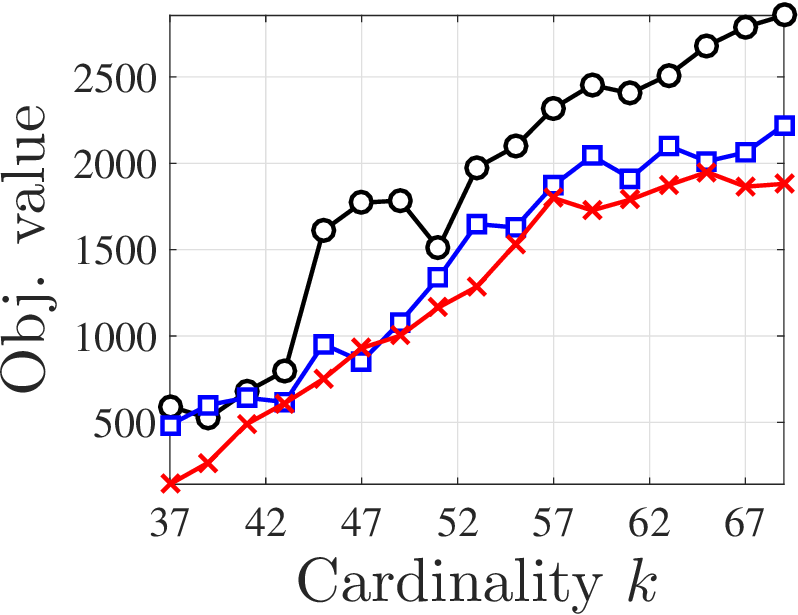}} &
\subfloat[MNIST ($\tau=25$)]{\includegraphics[width=3.8cm,height=3.2cm]{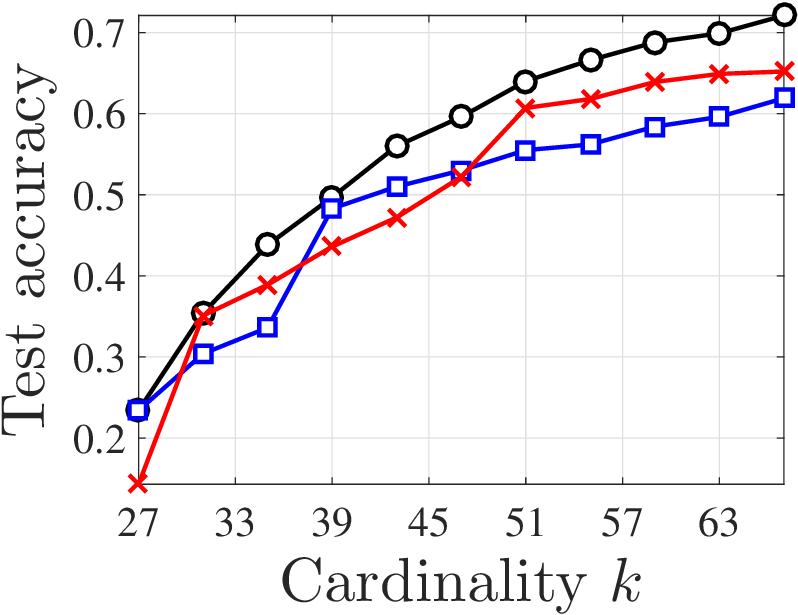}} &
\subfloat[MNIST ($\tau=35$)]{\includegraphics[width=3.8cm,height=3.2cm]{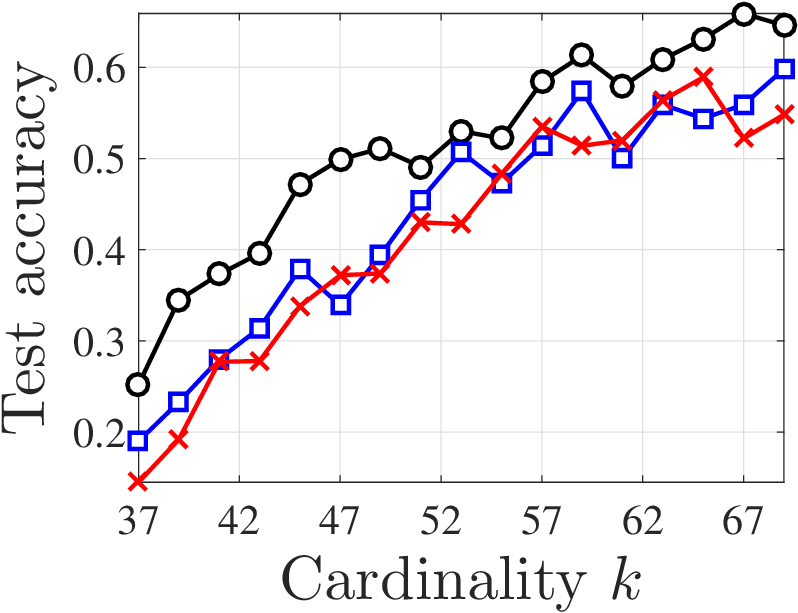}}
\end{tabular}
  \caption{Additional experiments for comparison of the algorithms on support selection task.}
\label{fig:lin_syn_fig}
\end{figure}

\end{document}